\documentclass[11pt]{article}

\usepackage[a4paper,margin=1in]{geometry}

\usepackage[T1]{fontenc}
\usepackage[utf8]{inputenc}
\usepackage{lmodern}

\usepackage{amsmath,amssymb,amsfonts}
\usepackage{amsthm}
\usepackage{bm}

\usepackage{graphicx}
\usepackage[export]{adjustbox}
\usepackage{multirow}
\usepackage{booktabs}
\usepackage{array}
\usepackage{longtable}
\usepackage{float}
\usepackage[section]{placeins}
\usepackage{xcolor}

\usepackage{caption}
\usepackage[title]{appendix}

\usepackage[numbers,sort&compress]{natbib}

\usepackage{xurl}
\usepackage[colorlinks=true,linkcolor=blue,citecolor=blue,urlcolor=blue]{hyperref}

\theoremstyle{plain}
\newtheorem{thm}{Theorem}[section]
\newtheorem{defn}[thm]{Definition}
\newtheorem{prop}[thm]{Proposition}

\theoremstyle{definition}

\theoremstyle{plain}
\newtheorem{theorem}{Theorem}

\newtheorem{learningproblem}[theorem]{Learning Problem}
\newtheorem{thesis}{Thesis}
\theoremstyle{remark}

\theoremstyle{definition}

\theoremstyle{plain}

\DeclareMathOperator{\CA}{\bm{\mathtt{CA}}}
\DeclareMathOperator{\SA}{\bm{\mathtt{SA}}}

\DeclareMathOperator{\mathworld}{\bm{\mathtt{MathWorld}}}

\definecolor{fixedBlue}{RGB}{50, 60, 90}
\definecolor{modOrange}{RGB}{200, 90, 20}
\definecolor{fixedFill}{RGB}{240, 245, 250}
\definecolor{dataGreen}{RGB}{40, 120, 80}
\definecolor{dataFill}{RGB}{230, 242, 235}

\newsavebox{\tablebox}

\title{Discovering mathematical concepts through a multi-agent system}

\author{%
Daattavya Aggarwal$^{1*\dagger}$,\ Oisin Kim$^{1*\dagger}$,\ Carl Henrik Ek$^{1,2}$ and Challenger Mishra$^{1}$\\[1ex]
{\small $^{1}$Computer Laboratory, University of Cambridge}\\
{\small $^{2}$Karolinska Institutet}
}

\date{\today}

\begin{document}

\maketitle

{\renewcommand{\thefootnote}{}%
\footnotetext{$^{*}$Corresponding author(s). E-mail(s): \href{mailto:da579@cam.ac.uk}{da579@cam.ac.uk}; \href{mailto:osmk3@cam.ac.uk}{osmk3@cam.ac.uk};\\ $^{\dagger}$Equal main authors.}}

\begin{abstract}
Mathematical concepts emerge through an interplay of processes, including experimentation, efforts at proof, and counterexamples. In this paper, we present a new multi-agent framework for computational mathematical discovery based on this observation. Our system poses its own conjectures and then attempts to prove them, making decisions informed by this feedback and an evolving data distribution. Inspired by the history of Euler's conjecture for polyhedra and an open challenge in the literature, we benchmark on the task of autonomously recovering the concept of homology from polyhedral data and knowledge of linear algebra. The system (which can be LLM-free) completes this proof-of-principle learning problem. Importantly, the experiments are ablations, statistically testing the value of the complete dialectic. Our system significantly outperforms ablated counterparts in using the relevant invariants, particularly those which are difficult to spot through naive pattern detection (e.g. Betti numbers). We also benchmark against modern LLMs. Despite their impressive problem-solving ability and extensive training in the relevant concepts, they struggle on a similar inductive reasoning problem. All this supports our main claim: that the optimisation of the right combination of local processes plays an important role in generating well-aligned notions of mathematical interestingness.
\end{abstract}

\tableofcontents

\section{Main}
\subsection{Introduction}
The main contribution of this paper is a multi-agent reinforcement learning framework for mathematical
discovery inspired by the following:

\begin{thesis} \label{thesis1}
Posing good questions is an essential aspect of mathematical practice; mathematical intelligence is not exhausted by the solving of fixed problems. Research mathematics involves an interplay between inductive and deductive reasoning; often, limits on what can feasibly be proved narrow the inductive search space.
\end{thesis}
This statement is hardly controversial. It is consistent with the insights of theories of mathematical practice \cite{Lakatos_Worrall_Zahar_1976}, and should also agree with the intuitions of most mathematicians. Nevertheless, it bears reiterating in the context of present efforts at automating mathematics. Although the ability to solve fixed problems is increasingly rapidly, there is still little evidence of AI's capability to pose and formulate `interesting' mathematical questions and definitions of its own, a search which can prove open-ended. 

Here, we present a framework for AI-assisted mathematical discovery based on the above thesis. Our system is not inherently tied to applications in a particular domain. However, we only test on the proof-of-principle task of `rediscovering' the concept of homology, which our system succeeds at.  The most interpretable result is that the Betti number, historically a proof-generated concept, is noticed at a significantly higher rate if conjecturing is constrained by provability. We also find that commercially available LLMs perform surprisingly poorly in this task (§\ref{llmsec}), especially considering that they have been extensively trained on the relevant concepts. In §\ref{motiv} we briefly motivate Thesis \ref{thesis1} and outline the structure of the paper.

\subsection{Motivation and outline} \label{motiv}
Mathematicians often emphasise the holistic nature of their subject, shaped by the interplay of questioning, concept formation, and answering. According to Alexander Grothendieck, `the decisive thing is often to see the question that had not been seen ... it does not matter whether the proof is trivial or not, which is entirely incidental, or even whether a hasty and provisional proof turns out to be false'. Indeed, `the simple act of writing, naming, and describing ... always precedes the proofs and gives us the means to prove' \cite{grothendieckRecoltesSemaillesReflexions2021}. These quotes underscore the creative power of mathematical questions and concepts, including their ability to unlock a solution. Conversely, it is common for attempts at proof to feed back into our choices of questions and definitions, since one constantly revises them based on this experience. 

Mathematics is a dynamic, dialectical process. At a given moment, new concepts, formalised as definitions, are created to address existing problems, and good answers, in the form of proofs, open up new fields of inquiry. This is a continuous evolution; for David Hilbert, `every age has its own problems, which the following age either solves or casts aside as profitless and replaces by new ones' \cite{hilbert1902mathematical}.

The purpose of these quotes is not to invoke timeless authority, but to re-examine the human practice of mathematics in the context of the AI-led case. We consider human mathematical research to improve the artificial one, which is advancing rapidly (§\ref{sec1.2}). It is still unclear whether current approaches can `research autonomously' (§\ref{refdrawbacks}), often requiring the formulation of new and humanly interesting ideas. This paper studies this ambitious computational problem, inspired by Thesis \ref{thesis1} and existing accounts of the subject (§\ref{sec1.3}).

Our main contribution is a new multi-agent AI system for mathematical discovery modelling the discipline as a holistic process. Conjectures produced through the analysis of data are constrained by their provability; this captures some aspects of the dynamic interplay between mathematical questioning and answering (§\ref{sec1.7} - \ref{sec1.12}). Although demonstrating the feasibility of combining the components is a major result of this paper, our central claim is that it is the nature of this interaction that is most important in the emergence of novel concepts.

To test this, we conduct experiments, designed to check whether our system can recover interesting notions from data and elementary knowledge. The benchmark is to autonomously formulate and meaningfully relate two definitions of the \emph{Euler characteristic}. In order to emulate historical conditions, the AI must do this by observing polyhedral data, and may only possess `knowledge' of linear algebra (§\ref{sec1.5} - \ref{sec1.6}). 
Our system completes this task, suggesting the `rediscovery' of the idea of \emph{homology}, within certain limits (§\ref{sec:results} - \ref{sec:Discussion}). In contrast, state-of-the-art commercially available LLMs perform poorly on the similar problem of noticing the concept of homology from the same polyhedral data (§\ref{llmsec}).

Crucially, the results are \emph{ablation studies} (§\ref{sec:chi_homology_baselines}). This means removing parts of the model to measure the value of the full interaction and to contextualise the idea of `rediscovery'\footnote{For example, by controlling for the innate capability of the regressor.}. A great challenge of assessing mathematical `interestingness' is that it can be very hard to spot and impossible to define, especially for new ideas (§\ref{refdrawbacks}). The chosen problem avoids this difficulty. Furthermore, a very similar open challenge already existed in the literature (§\ref{sec1.5}) \cite{Harris_2024}.

The value of proofs goes beyond checking the truth of propositions; mathematical questions often evolve in the course of attempting proofs. These ideas, also inspiring the design of our system, are shown elegantly in our first and main case study.

\subsection{Euler's conjecture for polyhedra} \label{sec1.4}
One of the most fascinating stories of mathematical discovery surrounds Euler's conjecture on polyhedra, once dramatised by Imre Lakatos \cite{Lakatos_Worrall_Zahar_1976, cromwell1997polyhedra, richeson_eulers_2008, polya1, alama}. At its centre are the concepts of \emph{homology} and the \emph{Euler characteristic} (denoted $\chi$). These topological \emph{invariants}, or signatures that encode information about shape, changed the study of spaces by allowing the use of new algebraic methods. 

Mathematicians now apply $\chi$ and homology in many ways, whilst outside the subject, they are used in fields such as data science and physics \cite{carlsson_topology_2009, green_superstring_2012}. But the Euler characteristic has surprisingly humble beginnings, in the study of polyhedra. During the $19^{\textnormal{th}}$ century, the quantity $\chi:=V-E+F$, where $V$, $E$ and $F$ denote the numbers of vertices, edges and faces, began to be studied seriously for the first time. It may seem surprising that despite thousands of years of interest in these surfaces, it took so long to define this\footnote{One possible reason is that the concept of an edge did not solidify until the time of Euler; of course, the alternating sum is also very difficult to naively motivate \cite{richeson_eulers_2008}.}.

In 1758, Euler conjectured $\chi=2$ for all polyhedra, testing his claim in a variety of cases. However, his thesis was soon refuted by many counterexamples, for example the picture frame (Figure \ref{fig:placeholder}). The problem was that another term was needed to account for `holes', since for ordinary polyhedra, the correct statement turned out to be $\chi=2-2g$. Here, $g$ denotes the genus, or number of `holes', which was eventually  defined independently using algebra applied to combinatorial constructions. In modern terms, this was homology\footnote{We will treat this more carefully in §\ref{sec1.6}.}.

In 1895, this revised equality was proven by Poincaré, when he showed that surfaces deformable to each other must possess the same $\chi$ \cite{poincare_papers_2010}. Looking forward, this will prove to be an exemplary case of the story outlined in  §\ref{sec1.3}, one in which $\chi$ and homology are concepts influenced by attempts at proof. What was perhaps unusual within mathematics was the centrality of a pattern detected from data, $V-E+F=2$. It has been argued that Euler used largely inductive reasoning to arrive at this expression, making it a natural candidate for computational investigation \cite{polya1}. 

\begin{figure}[!htbp]
    \centering
    \includegraphics[width=\linewidth]{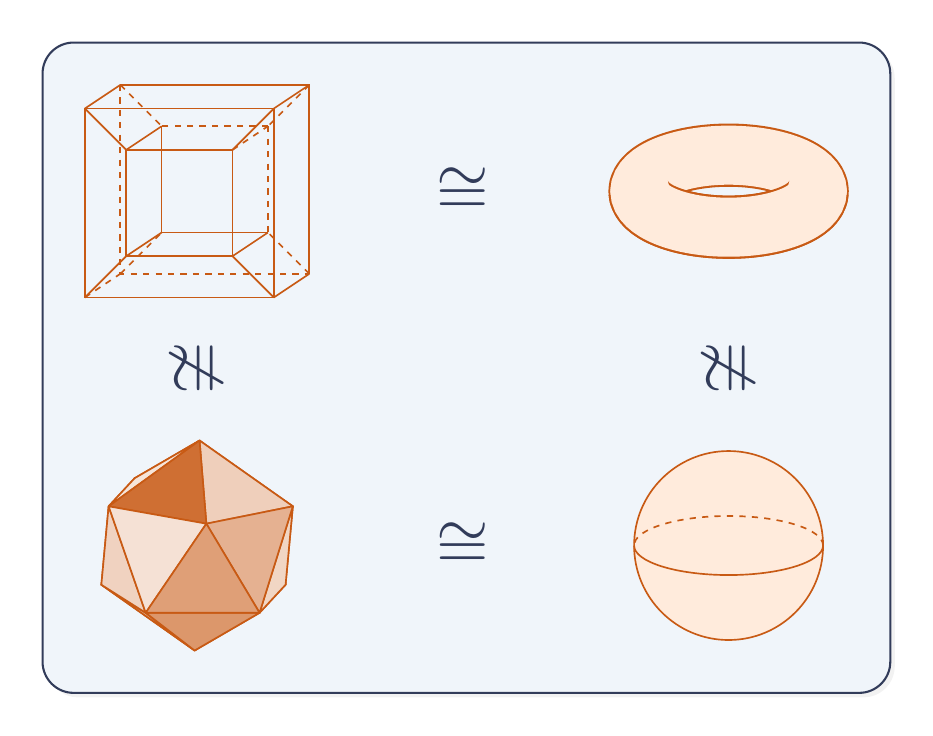}
    \caption{The picture-frame (introduced by L'Huilier) is topologically the same as the torus, but not the icosahedron, which has no hole. In modern language, the symbols denote \emph{homeomorphisms}, a precise notion of topological equivalence.}
    \label{fig:placeholder}
\end{figure}

\subsection{Mathematics and AI} \label{sec1.2}
The case of Euler demonstrates that thoughtful calculation has always played a crucial role in the subject\footnote{Including the most ancient cultures \cite{chemlaHistoryMathematicalProof2015}.}. New forms of computation continue this tradition, through symbolic and statistical approaches \cite{Hales_2014}. AI is being applied in mathematical question-posing and answering; these activities are sometimes formalised as conjecturing and proving. However, such tools have very rarely, if ever, been applied to these areas together, in a combined, iterative approach. 

A line is often drawn connecting past uses of machines in mathematics, such as the classification of finite simple groups, the four colour theorem, and the Kepler conjecture \cite{Appel1977429, Appel1977491}. Whatever the value of this historically, computers are now clearly indispensable to the subject. For example, symbolic computing such as Mathematica has assumed a key role since its invention, across a wide range of areas\footnote{From number theory to mathematical physics.}. Within the scope of `answering', proof assistants have led to the resurgence in interest in formal logic, in adjudicating the correctness of certain arguments and definitions, even in indisputably deep areas of maths \cite{lean-liquid}. Statistical tools can also assist mathematicians. They can conduct literature searches, as suggested by several recent cases of LLMs providing helpful but obscure theorems\footnote{These approaches can go very far \cite{openai_gpt5_math_discovery_2025}, extending much beyond the straightforward notion of a literature search.}.

Machine learning has been particularly useful for constructing mathematical objects with certain properties, when there is a direct way of measuring success. This lies, for instance, behind the FunSearch and AlphaEvolve algorithms, which use genetic programming to evolve LLM-generated objects, such as sets or functions, into new versions with higher fitness scores \cite{romera-paredesMathematicalDiscoveriesProgram2024, novikov_alphaevolve_nodate}. In many such cases, it is clear a priori that these problems can be posed as optimisation tasks; this should not downplay the ingenuity necessary to develop the techniques in any way. 

Advances in computation have changed mathematics and its values throughout history, even on an aesthetic level. The progress summarised above will likely accelerate this \cite{avigad_varieties_2021, venkatesh_thoughts_2024}. With an AI perspective, these developments naturally lead us to the prospects of a `general mathematical intelligence', capable of researching (somewhat) independently. Belief in this is partly based on an old idea: that the subject is a form of serious game, with participants using creative strategies to navigate rule-based systems \cite{VonNeumann1964-VONTFF}. AI, especially \emph{reinforcement learning} (RL), has proven to be very successful in such contexts. 

In 2018, DeepMind released AlphaZero, an RL approach to games like Go and chess \cite{silver_general_2018}. Most interestingly, it performed best when trained on its own trajectories, enabling the efficient exploration of the immensely large space of all possible games. Systems built on this approach have since produced impressive results in mathematics, e.g. AlphaProof, which gave a silver medal performance at the 2024 International Mathematical Olympiad (IMO) \cite{alphaproof, hubertOlympiadlevelFormalMathematical2025}. Here, DeepMind adapted their ideas to formal mathematics, modifying the system for symbolic expressions written in Lean and ruthlessly pruning proof trees of failures. The utility of computers as tools is indisputable; this perspective will continue to serve practising mathematicians well. However, there are difficulties arising from more ambitious applications, already hinted at by those that we have considered here and the scope of Thesis \ref{thesis1}.

\subsection{Current limitations of AI for maths} \label{refdrawbacks}
All machine learning needs objectives to optimise. It is unclear whether the qualities humans value in their mathematics can be directly formulated in this way, beyond restricted cases. 
 
Whilst an analogy with chess is compelling, discovery is difficult to formulate as a similar two-player game, having both competitive and collaborative aspects \cite{dutilh_novaes_dialogical_2020, thurstonProofProgressMathematics1994}. More broadly, applying AI requires encouraging the right behaviour. In the case of `questioning' activities, such as conjecture generation, a key aim is for the posed statements to be `interesting', as judged by mathematicians. But it is hard to directly encourage interestingness, and so far, the main approach has been for experts to provide considerable problem-specific intuition, for instance in low-dimensional topology or group theory \cite{davies_advancing_2021}. Naturally, this compromises generality. Others have suggested information-theoretic or geometric definitions, but these are currently at a purely theoretical or rudimentary level \cite{mishra_mathematical_2023, bengio_machine_2024}. 

In the case of `proving' activities, the search space is enormous, whilst basic results assert that the general decision problem is undecidable. When RL is applied, rewards are extremely sparse, since a single wrong step can totally derail an attempt \cite{shehper_what_2024, hubertOlympiadlevelFormalMathematical2025}. Work on automated formal provers is partly motivated by the belief that they can contribute to mathematical research \cite{avigad_mathematics_2024}. Where formalisation is already doing this, such tools can certainly help, for example by facilitating large collaborative  projects \cite{bolan_equational_2025}.  
 
Yet we do not believe that provers, considered on their own, qualify as autonomous mathematical intelligences. In competition mathematics, the computer is only required to produce a single proof of a fixed question within a predefined, narrow area. Even more recent startling research applications have not, at the time of writing, autonomously produced new human-understandable concepts with the clarity and significance of homology or a group. The greatest successes have involved solutions to long-standing fixed problems. In contrast, the notion of mathematical intelligence outlined in Thesis \ref{thesis1} is inherently fluid, often requiring the repeated non-trivial reformulation of questions. This frequently needs the creation of humanly interesting concepts, combining ideas from seemingly different mathematical areas, all in dialogue with proof.
 
The problems highlighted here stem from the same fundamental question: to what extent can optimisation, narrowly defined, produce `human' outcomes?

\subsection{Human mathematical practice} \label{sec1.3}

We briefly consider the human practice of generating mathematical ideas.

One picture of the subject is a tree, with axioms sitting at the base. Theorems are deduced by rules of inference, creating infinitely many branches. Of course, the overwhelming majority of true statements will not be interesting to humans; we can think of those that are as forming a very thin, branching connected path. Motivated by our interest in automation, we hope to better understand the local processes by which this subset is chosen.

Firstly, we believe that `interestingness' is conditioned on history, depending on which parts have already been explored. Concepts such as group or manifold are often data-compressing abstractions unifying previously disparate phenomena. Created to solve specific problems and clarified through applications, they become central through use and interaction. We are not claiming that mathematical concepts are solely shaped this way but that it plays a considerable role\footnote{Such a view would be particularly reductive in mathematical physics or complexity theory, where external factors are particularly important.}. Closely linked to this is the choice of questions, judged best when they open up new areas.

Many case studies support this narrative, making it almost hard to single out any in particular. For important concepts first invented to tackle particular problems, consider the Abel-Ruffini theorem, following centuries of attempts to generalise the quadratic formula to higher degrees. The proof required a sequence of new definitions: of a \emph{group} acting on roots, of the \emph{splitting field} of the polynomial, and of the \emph{automorphism group} of a field \cite{edwards_galois_1998}. Perhaps more consequential than the theorem itself were these innovations, which introduced a new level of abstraction. 

For the right choice of the question opening vistas, consider Hilbert's tenth problem, on the existence of an algorithm for determining whether an arbitrary Diophantine equation possesses integer solutions. A simple shift of emphasis - rather than seeking such an algorithm,  proving that it cannot exist - spurred considerable development in number theory, logic, and their intersection, leading to the solution and beyond\footnote{Of course, this took place against the background of wider developments in algorithms and logic.} \cite{Matijia2003}. It is striking that such a seemingly subtle change of perspective could prove to be so fruitful.

Although these examples have had historical importance, playing out over decades and centuries, we also believe that similar dynamics exist in small-scale problem solving such as the proving of a lemma. Naturally, documenting evidence is much harder to find in such cases.

To summarise, good mathematical ideas emerge from an interplay of local processes, some of which have been highlighted here. Our central view is that mathematical value, like many other human practices, is best understood within a complex network of pressures - such as proof, experimentation, and refutation -  making it difficult to directly quantify or optimise for in isolation\footnote{Similar ideas to what we have discussed in this section have been mentioned by many others \cite{ gowers_how_nodate, freedman2025poincare, taoWHATGOODMATHEMATICS}.}.

\subsection{Euler's conjecture from a statistical perspective}
\label{sec1.5}
In §\ref{sec1.4} we introduced Euler's conjecture, which fits neatly with the above. The question passed through many stages, evolving through the alternating pressures of counterexamples and proof. With an ML perspective, it can also be thought of as a learning problem. The `datapoints' of surfaces gave Euler statistical evidence for the conjecture $\chi=2$, then reshaped through the process of §\ref{sec1.3} into a theorem containing deep definitions.

We will therefore treat the reformulation of the concept of homology as a test for an AI system. It should do this with knowledge that was available at the time of its discovery, to the extent possible. It is also notable that $\chi=2$ was first detected in the restricted case of surfaces homeomorphic to the sphere.

The Euler characteristic formula has many mathematical subtleties. For example, it is only if written as an alternating sum, i.e. $V-E+F=2$, rather than $V+F=E+2$, that it truly suggests its topological aspect\footnote{ This was even missed by Euler \cite{euler1758elementa}.}. It is likely harder to imagine giving a careful definition of homology exclusively through statistics. One of the most challenging aspects of detecting $\chi$ and homology this way is their `local to global' nature; they measure global features of the space from local information (connections between vertices and edges).

A recent essay by the arithmetic geometer Michael Harris states the following challenge: `Give AlphaGeometry 2.0 the list of [Euclidean] axioms and let it generate its own constructions for as long as it takes. Then wait to see whether it rediscovers the Euler characteristic', simply `by analyzing the statistical correlations', as part of a list of questions for a new mathematical AI \cite{Harris_2024}. We did not know about this open challenge when we conducted the experiments in this paper, but were delighted to find it whilst writing.

In §\ref{sec1.6} we present the modern concepts necessary to precisely state the goal and setup of the learning problem; we assume some familiarity with elementary topology. See \cite{hatcher_algebraic_2001, gallierGuideClassificationTheorem2013, dean2014homology} for further details and proofs (the last is understandable with just linear algebra).

\subsection{Learning problem and mathematical background} \label{sec1.6}
Our experimental data is a collection of matrices, each containing information about a surface. This is similar in spirit to Euler's analysis of polyhedra, but with more examples and a somewhat modern presentation\footnote
    {We discuss the implications of this and the other choices of this section in §\ref{sec:Discussion}.}.

More precisely, we use incidence matrices associated to a collection of \emph{simplicial complexes} (SCs). These are sets of triangles, edges and vertices (simplices) glued together `nicely'. For a basic picture of a 2-dimensional simplicial complex, think of a polyhedron with triangular faces. Our choice of data is justified by the fact that all closed surfaces (two-dimensional manifolds without a boundary) admit a triangulation\footnote{This means dividing the whole surface `nicely' into connected triangles.}, and that topological properties are invariant to this choice. For the remainder of this paper we specialise to this setting.

Consider the vector space over a 
field $k$ (with characteristic $\neq2$) freely generated by the $i$-simplices (simplices of dimension $i$); denote this space by $C_i$. We define the maps $\partial_i: C_i \rightarrow C_{i-1}$ by sending a simplex to the sum of simplices on its boundary, where the coefficients $\pm1$ respect orientations. It is easy to show that these maps satisfy $\partial_{i-1}\circ\partial_{i}=0$, an important property for the definitions that follow\footnote{`The boundary of a boundary is zero' \cite{Wheeler1989-WHEIPQ}.}.

Given a triangulated surface, the topological information is encoded in a pair of \emph{incidence matrices}, corresponding to the non-trivial $\partial_i$ (i.e. $\partial_1, \partial_2$). The vector spaces above fit into the sequence:

\begin{equation}
 0 \rightarrow C_2\xrightarrow{\partial_2} C_1 \xrightarrow{\partial_{1}} C_0 \xrightarrow{\partial_{0}} 0.
\end{equation}
After we have triangulated a polyhedron, the numbers $V$, $E$ and $F$ become equal to $\textnormal{dim}(C_0)$, $\textnormal{dim}(C_1)$, and $\textnormal{dim}(C_2)$. It is easy to show that during this process $V-E+F$ remains unchanged. 

This is the remarkable part: in modern terms, the (simplicial) homology is defined as the quotient spaces $H_i :=\textnormal{Ker}(\partial_i)/\textnormal{Im}(\partial_{i+1})$.\footnote{This ignores \emph{singular homology}, which was developed later and shown to be equivalent \cite{hatcher_algebraic_2001}. Both forms of homology count the number of bounding \emph{cycles}, though we do not carefully define this.} The $k$-dimension of $H_i$ is denoted $b_i$, and is called the $i^{th}$ \emph{Betti number} (with respect to the field $k$). 

We use the following very standard results \cite{hatcher_algebraic_2001}:
\begin{prop}
$b_0$ counts the number of connected components.
\label{prop:b0}
\end{prop}
\begin{prop}
For a connected surface, $b_2$ is $1$ if and only if the surface is orientable.
\label{prop:b2}
\end{prop}

It turns out that the other Betti number, $b_1$, is equal to twice the genus, or number of geometric `holes'. Our common-sense idea of a `hole' is ambiguous, which caused difficulties in mathematics until both it and `polyhedra' were sharpened by being incorporated into the framework described here \cite{Lewis01081970, Lakatos_Worrall_Zahar_1976}. 

We can now capitalise on this work. Euler first defined $\chi$ as the alternating sum of the vertices, edges and faces of a polyhedron. Now we extend this to all surfaces:
\begin{defn} \label{thm1}
Given a triangulation of a compact surface, the Euler characteristic is defined as $\chi:=\textnormal{dim}(C_0)-\textnormal{dim}(C_1)+\textnormal{dim}(C_2)=V-E+F$. 
\end{defn}

A crucial discovery was the following alternative:
\begin{defn} \label{thm2}
Given a triangulation of a compact surface, the Euler characteristic can also be defined as $\chi:=b_0-b_1+b_2$.
\end{defn}
A priori, it is not clear that these different formulations are equivalent. Proving that they are was a major breakthrough, generally credited to Poincaré. Although the older definition \ref{thm1} may make the topological invariance of $\chi$ seem rather mysterious, if we know that the $b_i$ are themselves invariants\footnote{It is likely that Poincaré did not know this to a modern standard of rigour; the proof is most natural with singular homology.}, the new definition makes this clear. It pointed the way forward for modern topology by giving old concepts a novel algebraic interpretation. 

Returning to Euler's conjecture, we have the following important reformulation:
\begin{thm} \label{thm4}
For a closed surface, if $b_0 = 1$, $b_1 = 0$, and $b_2 = 1$, then $\chi := V-E+F=2$.
\end{thm}

This leads to the main benchmark task of this paper.
\begin{learningproblem}
\normalsize
Can an AI system recover the definitions and relationship between \ref{thm1} and \ref{thm2} from data and knowledge of linear algebra, without being explicitly coached to do so?
\label{lp1}
\end{learningproblem}

In our dataset, $V$, $E$, and $F$ correspond to heights and widths of matrices, and $b_i$ can be expressed in terms of ranks and nullities. So it should be possible for a well-designed system to succeed without any prior direct encoding of the concepts. A relationship between definitions \ref{thm1} and \ref{thm2} must be generated without the expressions for $b_i$ or $\chi$ being provided in any way. For example, recovering Theorem \ref{thm4} would complete the challenge, which is essentially what happened historically.

The relevant information is contained in the incidence matrices, but only implicitly. In order for an AI to produce a statement of the required form, it will be necessary for it to come up with the relevant concepts. In contrast to many of the situations of §\ref{sec1.2}, there is no clear way of thinking about this process just in terms of overcoming a single fixed problem. As in the examples of §\ref{sec1.3}, the system should learn to identify the value of the concepts for itself. 

The types of surfaces used were $(i)$ triangulated spheres $(ii)$ triangulated tori $(iii)$ triangulated Klein bottles $(iv)$ disjoint unions of the previous (in this case, we have block diagonal incidence matrices). We constructed the data using Delaunay triangulations and gluing. The code samples random triangulations of these surfaces, and also supports higher genera and non-orientable spaces.

\subsection{Reinforcement learning} \label{sec1.7}
Drawing inspiration from the previous discussion, we present a new multi-agent AI system for mathematical discovery. Reinforcement learning provided a plausible framework within which the pursuit of simple goals can produce rich behaviour through exploration and competition. For more background, see \cite{suttonReinforcementLearningIntroduction2020}.

The formal setup is an \emph{environment}, modelling a real situation as a tuple $M=(S, A, R, P) $. This is a state space, action space, rewards, and transition probabilities, collectively known as a Markov Decision Process (MDP). In the finite case, one can think of a graph of nodes corresponding to states, with edges labelled by probabilities and rewards. 

The fundamental loop involves an \emph{agent} choosing an action $a$ in a state $s$, which transitions it to a new state $s'$ via the probability $p(s' |s,a)$, resulting in a reward $r(s,a,s')$\footnote{This can also be probabilistic.}. All of this is strictly determined by $M$. For example, in chess $S$ is naturally taken to be the space of board positions and $A$ individual moves. Then $P$ is deterministic and $R$ is only allocated at the end of games. Given an MDP, the sequence $\{s_0, s_1,..\}$ is a Markov process. We call the position within this the \emph{timestep} $t$.

Starting from $s_0$, the agent aims to maximise its reward $\sum_{t=0}^\infty \gamma^tr(s_t, a_t, s_{t+1})$, where $\gamma\in(0,1)$ is a discount factor (assigning less weight to future rewards). It can do this by choosing the right \emph{policy}, a distribution $\pi: S \rightarrow A$. There is often a distinguished $s_T\in S$, called the \emph{termination} state, which, if it is reached, ends the whole process. A full run $\{t_0,...,t_N\}$ is called an \emph{episode}.

MDPs find applications in economics and other areas. Under certain assumptions\footnote{For example, finite state and action spaces and bounded rewards.}, one can compute optimal policies with linear programming, assuming full knowledge of the MDP. However, in complex settings, where transitions are 
unknown and the $S$ and $A$ may be infinite or continuous, 
data-driven approximation techniques become necessary. This is really what is referred to as reinforcement learning. For instance, one can parameterise policies and expectations of future reward using neural networks.

During training, the agent samples through many possible actions and states, learning a policy that approximately maximises its accumulated reward. Next, actions are sampled exclusively from this in a final episode, called the \emph{execution}.
In multi-agent reinforcement learning (MARL), multiple actors act within the same MDP, introducing new challenges \cite{marl-book}.

It is perhaps surprising that such a simple formalism has achieved so much in recent years. 

\subsection{Introducing the system}

Our system uses two agents, called the \textbf{Conjecturing} and \textbf{Skeptical} agents ($\CA$ and $\SA$), to model the feedback loop between mathematical questioning and answering. The main goal of the $\CA$ is to produce a statement that is `provable enough'; the goal of the $\SA$ is to prevent this. The $\CA$'s statements are drawn from $\mathcal{S}$, denoting the space of all statements producible with a set of variables and operators\footnote{This includes false or even ill-formed examples.}. The environment, $\mathworld$, models the objects and proofs encountered in a particular phase of research. Concretely, this is a set of mathematical data $\mathcal{D}=\{d_i\}$ and a way of measuring provability $\rho:\mathcal{S}\rightarrow \mathbb{R}$. 

During an episode, the $\CA$ generates a sequence of statements by analysing the data. Meanwhile, the $\SA$ controls which parts of $\mathcal{D}$ the $\CA$ is allowed to see. Each statement is scored according to $\rho$. If one exceeds a provability threshold, the episode ends, with the $\CA$ receiving its largest reward and the $\SA$ receiving an equally large negative reward. The system must typically go through many false conjectures to reach this point. The $\CA$ also receives a small reward for longer statements, a deliberately simple proxy for complexity. Finally, if $t$ reaches fifty, the episode terminates anyway.

\subsection{Conjecturing agent} \label{sec1.9}
The role of the Conjecturing agent is to model mathematical questioning activities, producing a series of candidate statements. These should be mathematically interpretable - this can be non-trivial - and generated by analysing data, ideally in a manner that encourages the local exploration of diverse features. 

The canonical example of a mathematical question is the \emph{conjecture}, often codifying a long-standing open problem. These may drive entire research programs, requiring significant innovation for their solutions. We use the idea of `conjecture' more loosely, in line with the AI literature. The important thing is that $\CA$ poses its own questions, basing this choice on past experience. 

We found that a custom symbolic regression (SR) algorithm largely fulfilled these requirements. SR detects structure in a dataset by fitting analytic expressions to it, and is particularly useful with preexisting knowledge about what form this is likely to take. In our system, the $\CA$ controls regressors, producing statements $s$ by minimising a `loss' from $\mathcal{S}$ to $\mathbb{R^+}$:
\begin{equation} \label{regressorLoss}
\mathcal{L}(s) \propto \textnormal{exp} \bigg( \mathcal{L}_\mathcal{D}(s)-\sum_kp_k \bigg),
\end{equation}
where $\mathcal{L}_\mathcal{D}$ is a `data loss' and the $p_k$ are priors, all of which will be described below\footnote{Note that the exponent accentuates selection pressures and preserves positivity.}. In summary, this expression trades off fidelity to the data against inductive biases on the functional form.

In all SR, preexisting knowledge is incorporated by selecting its symbols: features and operators. As alluded to in §\ref{sec1.6}, in our experiments these will be:
\begin{enumerate}
    \item Operators: arithmetic binaries ($+,-, \times,=$) and generic logicals ($\land, \implies,\neg $)\footnote{These were chosen to form a complete basis in the sense of propositional logic.}. 
    \item Variables: the ranks, nullities, and dimensions of pairs of incidence matrices, each corresponding to a simplicial complex.
\end{enumerate}
Note that \textbf{whilst $b_i$ and $\chi$ can be expressed in terms of these, they are not provided to the SR}. To be spotted by the system, they must be formulated through regression. 

The choice of variables and operators is domain dependent; for example, in number theory one could use the number of prime divisors or arithmetic functions such as the totient or Möbius functions. In general, we imagine that most datasets impose a fairly natural choice for these. In a particular constructed expression there is no need for all of the symbols to be used.

With this setup, SR explores $\mathcal{S}$ by parameterising statements as trees, whose nodes are operators with one or two children, or features with none. Given a candidate expression, e.g. \begin{equation}\label{candidateeqn} \tilde{s}:=\{height(\partial_1)-width(\partial_1)+width(\partial_2)=2\},\end{equation} we first consider the data-dependent term $\mathcal{L}_\mathcal{D}$. The expression is true or false on a datapoint $d_i$, so $\tilde{s}(d_i)$ is a Boolean quantity. For example, Eq. \ref{candidateeqn}  will evaluate to $1$ on topological spheres and $0$ on topological tori. The weighted accuracy and full data-dependent terms are: 
\begin{equation}\label{ref3} \mathcal{W}_\mathcal{D}(\tilde{s})=\frac{\sum_i\lambda_i \tilde{s}(d_i)}{\sum\lambda_i}, \; \; \; \mathcal{L}_\mathcal{D}(\tilde{s})\propto \textnormal{exp}(\alpha(1-\mathcal{W}_\mathcal{D}(\tilde{s}))),\end{equation} 
where the $\lambda_i$ control attention given to each datapoint (we will return to this), and $\alpha$ is a hyperparameter. The (positive or negative) priors, $p_k$, in Eq. \ref{regressorLoss} encourage or discourage the use of operators or features, as well as ways of composing them\footnote{For example, to penalise particular errors.}.

For the actual statements, we use a two-step process, intended to reflect the structure of mathematical expressions. The first stage makes multiple atomic formulae (informally, features), each produced by a separate regressor. We call these first regressors the `Feature spotters', which learn from a specific local patch of the data $U\subset \mathcal{D}$, parameterised by coefficients $\lambda_i^U$\footnote{We will often suppress the patch notation in the $\lambda_i$}. Note that Eq. \ref{candidateeqn} could be produced from a patch highly weighted towards topological spheres.  The atomic formulae are then treated as variables by a final regressor, called the `Scaffolder', generating a full logical expression. This decomposition simplifies the search space, encouraging the exploration of features which are combined into a global structure.

The state space of the $\CA$ consists of the weights $\lambda_i$, priors $p_k$, and the number of features to generate. At each timestep, the $\CA$ updates these control parameters, deliberately biasing regression towards statements with a certain form.
This is summarised in Figure \ref{fig:SP}.
We used PySR to implement it \cite{cranmerInterpretableMachineLearning2023}, an SR package designed for scientific applications. Its main advantage was considerable flexibility in defining custom operators, constraints, and objective terms.

\begin{figure}[!htbp]
    \includegraphics[
        width=\textwidth, 
        center, 
        keepaspectratio
    ]{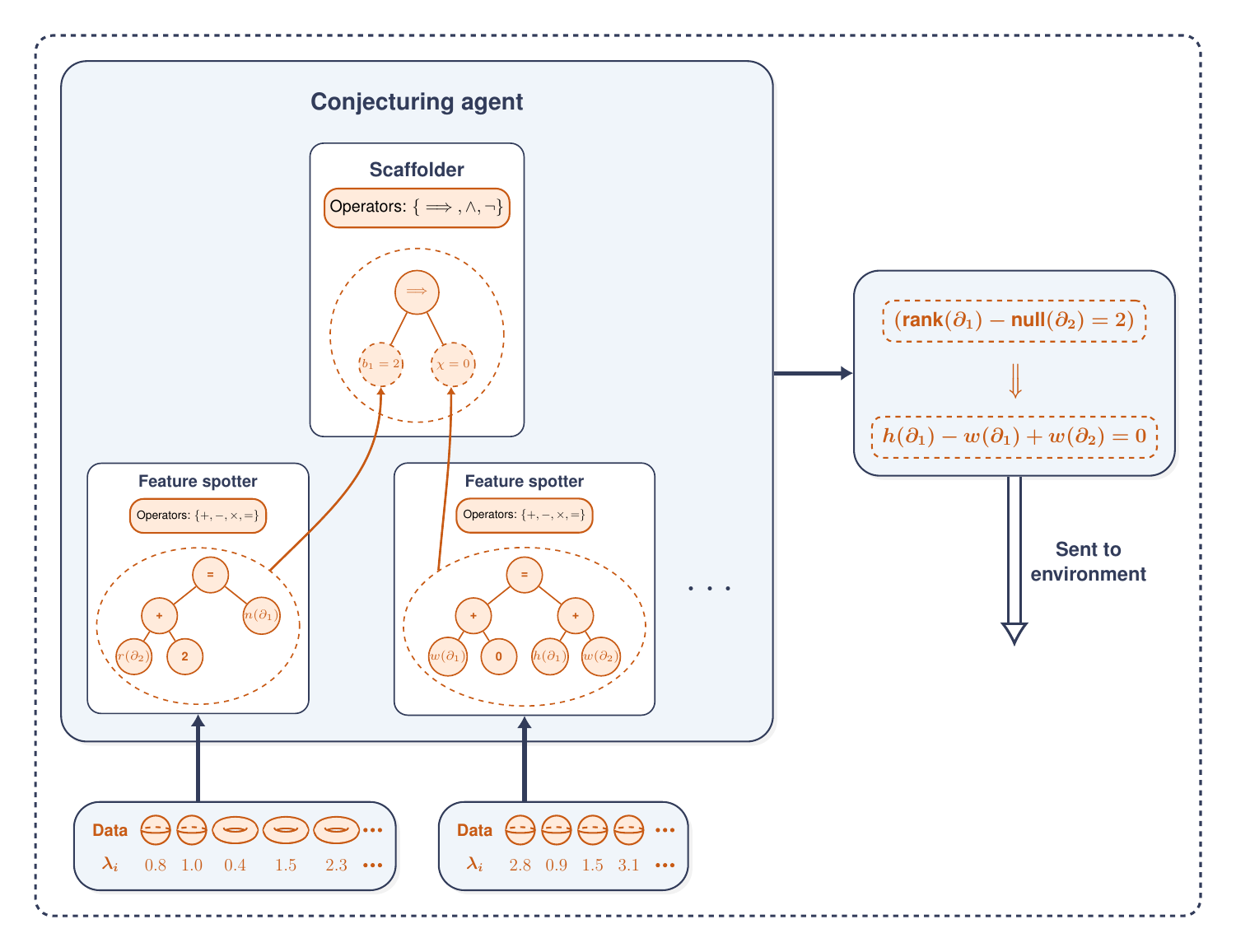}
    \caption{ 
    Summary of the Conjecturing agent. At each timestep, it picks its priors and chooses the number of features, then runs the first round of regression. This produces atomic formulae over patches of data, each governed by $\lambda_i$. In the second round of regression, these are combined by the scaffolder into a global statement. The variables and operators used are tailored to the application; a successful instance of learning is shown in the box on the right, with the heights and widths corresponding to the dimensions of vector spaces.
    Once in the environment, the statement gets translated into Lean. Quantifiers are implicitly added during the translation; the regression only has the variables and propositional symbols in its vocabulary. The Conjecturing agent cannot access the Lean statement. For further details, see §Appendix \ref{app:LeanTranslation}.
    }
    \label{fig:SP}
\end{figure}

\subsection{Skeptical agent}
\label{sec:1.10}

Beyond revising questions and using proof feedback,
another process is useful for our system to capture. In mathematics, partial attention to `data' can enable the discovery of patterns. 

$\chi=2$ was first spotted for polyhedra homeomorphic to spheres, with tori only noticed as counterexamples afterwards. 

Another case is the Riemann-Roch theorem, which expresses the dimension of the space of meromorphic functions on a complex surface in terms of topological information. Riemann initially proved this result for the canonical bundle\footnote{More specifically, he showed that the number of holomorphic $1-$forms is equal to the genus.}. It was only later, working with others, that he allowed singularities, giving the statement for all `datapoints' of Riemann surfaces \cite{riemann_theorie_1857, Roch+1865+372+376}. 

This is not just highlighting generalisation. Instead, restricting focus to a portion of mathematical data can make patterns easier to detect, which may then be reinterpreted and extended to a wider collection of objects. In the context of AI, we may think of this as an \emph{evolving data distribution}.

We model this through the $\lambda_i$ coefficients (see Eq. \ref{ref3}), changed dynamically by the Skeptical agent. Each feature spotter has its own set of these, and the scaffolder sees a union of datapoints across patches (taking the maximum $\lambda_i^U$ if the same $d_i$ is visible to multiple). At each $t$, the $\SA$ chooses a subset of the $\lambda_i$ to update, within fixed bounds. One can imagine it turning on a collection of counterexamples to the $\CA$'s first conjecture as learned behaviour. At the beginning of an episode, a fraction are initialised randomly, and the remainder are set to $0$\footnote{Recall that no tori were initially `visible' to Euler.}.

The goal of the $\SA$ is the opposite of the $\CA$'s. This competitive game produced good results and reflects aspects of mathematical practice. A partially cooperative variant may be more true to life, and could easily be implemented with the setup that will be discussed next. However, in reality, the discount factor $\gamma$ dampens the $\SA$'s opposition as an episode progresses. Since many episodes produce proven results (following training), we think of the $\SA$ as preventing an early and trivial statement terminating an episode too soon.

\subsection{Environment and provability} \label{sec1.11}
$\mathworld$ contains $(i)$ mathematical data $\mathcal{D}$ and $(ii)$ a provability measure $\rho:\mathcal{S}\rightarrow \mathbb{R}$.

A good dataset, chosen by the user, should contain a partially unknown and potentially interesting mathematical structure. We mean something somewhat specific by this: in the case of Euler, tables of $V, E,$ and $F$ arranged appropriately, or in the modern case of the BSD conjecture, tables of local and global invariants of elliptic curves\footnote{This narrative has been mentioned many times in the context of computers and maths. Naturally, in a typical application, the structure may be much less significant than BSD or $\chi$. \cite{buzzardAIMath}} \cite{euler1758elementa, birchNotesEllipticCurves1965}.

The provability score models the proving side of mathematics, assigning $\rho(s)$ to measure how close $s$ is to being `provable' with the system's knowledge. In general this may not be well defined, but we can consider simple approximations. Given an automated prover $A$, let:

\begin{equation}
\rho(s) :=
\begin{cases}
\text{1}, & \text{if $A$ finds a proof for $s$ before timeout}, \\
\text{0}, & \text{else}.
\end{cases}
\label{eq:discrete_proof_metric}
\end{equation}
In most of our experiments we used this discrete metric. 

One choice for $A$ is an LLM prover, e.g. \emph{Lean Copilot} \cite{song_lean_2025}. This attempts to prove the posed statements using a set of relevant theorems as background, such as the rank-nullity theorems for the $C_i$. For translation into Lean, we used a handcrafted code, with a header and premise set containing what we wanted $A$ to `know'; this prevented any hallucination effects caused by autoformalisation \cite{wu2022autoformalization}. 

Since our experiments mostly took place within linear algebra, we were also able to prove any correct statements with a prewritten proof, in which case the system is LLM free\footnote{This was done through a combination of tactics including $\mathtt{grind}$ and $\mathtt{ring}$.}. The tests produced the same qualitative results with either approach. We also tried some simple alternatives to $A$, discussed further in §\ref{sec:chi_homology_baselines}.

For a statement to obtain any reward or trigger termination, it must pass simple non-degeneracy checks. As an example, we discourage formulas that are true on all of $\mathcal{D}$, adding those that are noticed to the premises\footnote{Such as the rank-nullity theorems.}. These are somewhat discretionary choices, but the relationship between the knowledge of $A$ and the data should have a large impact on performance.

It is important that even a score of zero informs subsequent questioning. Unlike RL-based formal provers such as AlphaProof, we do not treat the prover as an agent, having it instead as part of the environment.

\subsection{Summary and optimisation} \label{sec1.12}
Figure \ref{fullsystemfig} summarises the system and shows the flow of information during an episode. Although components of the framework may be changed in a particular application, for example the symbolic vocabulary of the $\CA$ or the choice of prover, the overall structure should be domain invariant.
 
Optimisation was through multi-agent deep deterministic policy gradient (MADDPG) \cite{loweMultiAgentActorCriticMixed2017}. The central challenge of many agents is that, from the perspective of one, the environment also contains all the others, violating stationarity. MADDPG addresses this with \emph{centralised} training and \emph{decentralised} execution. During training, each agent can see the states and actions of others and estimates their value functions\footnote{This is the estimate of the accumulated reward given a policy and state.} using this knowledge. This recovers the MDP formalism\footnote{Importantly, this allows the use of the policy gradient theorem \cite{loweMultiAgentActorCriticMixed2017, SuttonPolicyGradient}.}. In execution, agents act  independently. The learned policies should be empirically `robust' to the opponents' actions, but in such complex games guarantees of reaching equilibria rarely exist \cite{marl-book}. 

\begin{figure}[!htbp]
    \includegraphics[
        width=\textwidth, 
        center, 
        keepaspectratio
    ]{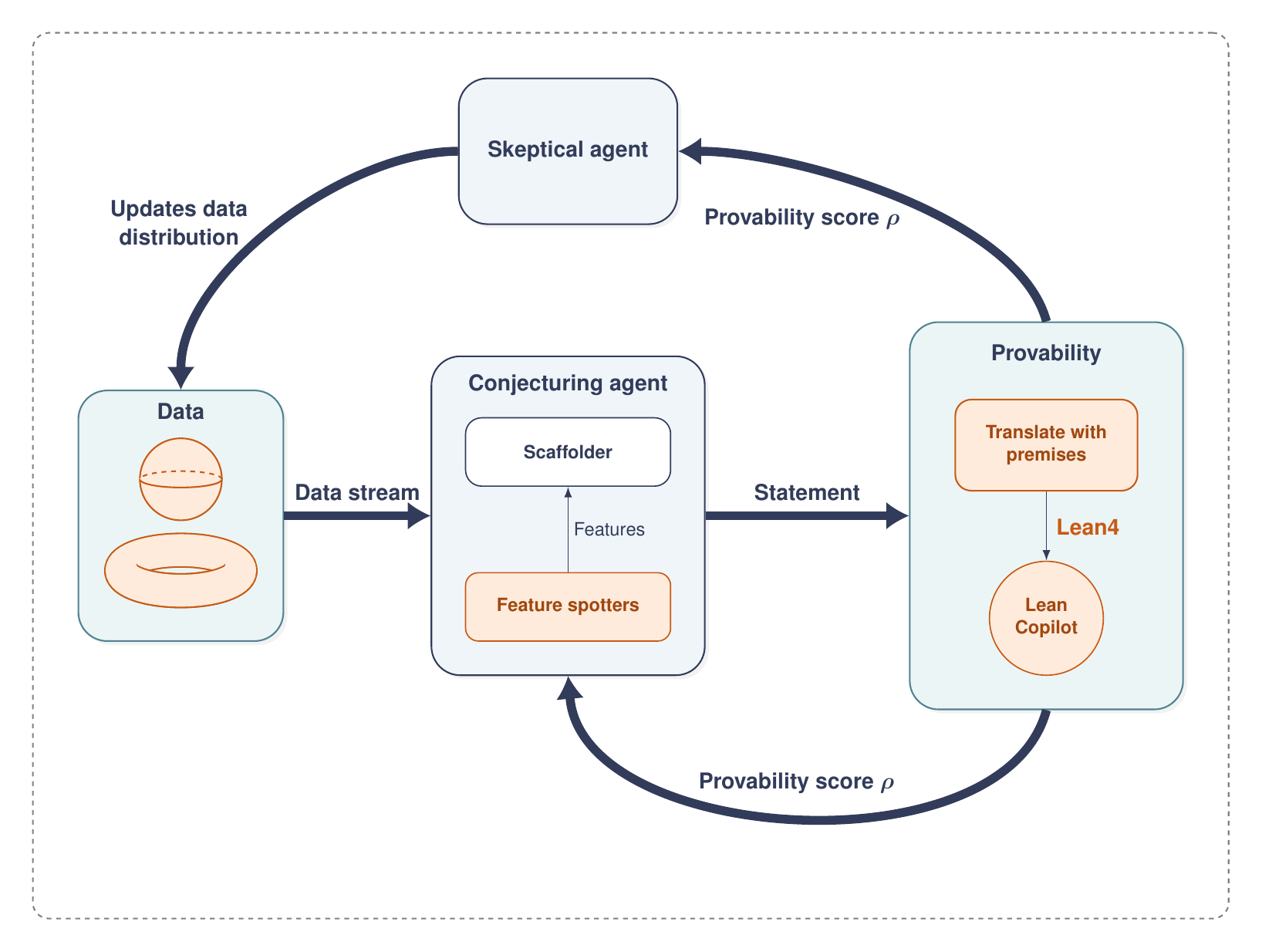}
    \caption{The Conjecturing agent ($\CA$) produces statements until one crosses the provability threshold whilst passing simple checks, at which point the episode terminates and it receives its greatest reward. At intermediate steps, it receives a smaller reward for longer statements, measured in terms of tree size. The Skeptical agent ($\SA$) modifies the data distribution. Provability is measured by a score $\rho$, for example calculated by passing the statement to an automated prover $A$ in the environment, e.g. Lean Copilot (§\ref{sec1.11}).
    The colour
    \textcolor{fixedBlue}{\textbf{blue}} is used for the agents and the flow of information, whilst \textcolor{dataGreen}{green} is used for elements that are in the environment. \textcolor{modOrange}{\textbf{Orange}} represents components that should be tailored to an application.  Although the game is competitive, it is not zero-sum. 
    }
    \label{fullsystemfig}
\end{figure}
\section{Results}\label{sec:results}
Recall that our benchmark task was to formulate and relate two definitions of the Euler characteristic from knowledge of linear algebra and data. Ultimately, we would like to apply this system in new areas, and the modularity of the components was designed with this in mind. The advantage of a known problem is that we already know what the humanly `interesting' concepts are, allowing easy interpretation of the results and a proof-of-concept. 

Given an interesting statement in  $\mathcal{S}$, we can reasonably assume that a collection of neighbours express the same meaningful content (with slight variations). So we pose our problem as recovering a relation between two definitions of $\chi$, rather than a particular example of this, such as Theorem \ref{thm4}. 

\subsection{Premises and data} \label{sec:2.1}
Having chosen $\CA$ as a regressor, it must learn the right symbolic expressions. But the exploration of the combinatorially growing space $\mathcal{S}$ poses many challenges\footnote{We will later demonstrate this by running experiments with only the regressor.}. 

Creating a statement with $\chi$ given through Betti numbers requires the system to notice these. For example:
\begin{equation} \label{thm3}
  \{b_0 = 1\} \land \{b_2 = 1\} \land \{b_1=0\}\implies \{V-E+F=2\},  
\end{equation}
can only be sensibly found if $b_0$ and $b_2$ sufficiently vary over $\mathcal{D}$. 

To test in a variety of settings, we use the following datasets (and associated premise sets, $\mathcal{P}$) for different experiments:
\begin{enumerate}
    \item $\mathcal{D}_0$: \{spheres, tori\} with $\mathcal{P}_0$: \{rank-nullity theorems\}, $\mathcal{P}_1: \{b_0=1\}$ and $\mathcal{P}_2: \{b_2 = 1\}$,\\
    \item $\mathcal{D}_1$: \{spheres, tori, Klein bottle\} with $\mathcal{P} = \{\mathcal{P}_0,\mathcal{P}_1\},$\\
    \item $\mathcal{D}_2$: \{spheres, tori, disjoint unions of previous\} with $\mathcal{P}=\{\mathcal{P}_0 \},$\\
    \item $\mathcal{D}_3$: \{spheres, tori, Klein bottle, disjoint unions of previous\} with $\mathcal{P}=\{\mathcal{P}_0\}.$
\end{enumerate}

These surfaces are some of the basic building blocks in topology, combined in equal numbers ($50-200$ of each over different experiments) to ensure balance. The \emph{classification theorem} says that any connected, compact surface without boundary is homeomorphic to a sphere or a connected sum of tori or projective planes \cite{gallierGuideClassificationTheorem2013}. The names are up to homeomorphism, so different tori or spheres may have considerable variation in $V$, $E$ and $F$. Our construction method meant that the distributions of these  were approximately normal.

For $\mathcal{D}_0$ we include $b_0=1$ and $b_2=1$ (representing connectedness and orientability) in $\mathcal{P}$, and just search for statements like $\{{b_1=0}\implies\chi=2\}$. For $\mathcal{D}_2$ the only premises are rank-nullity theorems, so we may hope for a conjecture similar to Eq. \ref{thm3}. Even if premises involving $b_i$ are used by $A$, the only effect is through $\rho$ feedback. In a new problem, finding an appropriate relationship between $\mathcal{P}$ and $\mathcal{D}$ may require experimentation.

We say that a model has noticed $\chi$ and/or $b_i$ if the $\CA$ makes a statement with an atomic formula containing either\footnote{Up to allowed substitutions from the premise set.}:
\begin{align*}
    \chi &: \text{dim}(C_0) - \text{dim}(C_1) + \text{dim}(C_2),\\
    b_i &:  \text{null}(\partial_{i}) - \text{rank}(\partial_{i+1})\quad i\in\{0, 1, 2\}.
\end{align*}
Proven conjectures with $\chi$ and/or $b_i$ are a weaker proxy for using the concepts than statements which clearly relate them meaningfully. This is useful to analyse.

\par
We standardise our experiments, fixing the total number of steps during training and fifteen episodes during evaluation. All training and evaluation experiments were conducted on an ordinary laptop, an Apple M-series ARM64 CPU without any dedicated hardware.

\subsection{Different models}\label{sec:chi_homology_baselines} 

To measure the difficulty of the Learning problem and the impact of each part of the system, we frame most of our experiments as ablation studies.
This means comparing the performances of different system versions, each with a different component removed, against each other, and crucially against the full case \cite{meyes_ablation_2019}. Ablations test our main claim: the importance of a full dynamic for generating valuable mathematical concepts.

The models used are:
\begin{itemize}
    \item Only $\CA$:  Just the Conjecturing agent. Without the $\SA$, every datapoint has a fixed weight of $\lambda_i=1$. The point is to study the intrinsic difficulty of the Learning problem; since this has been posed as finding the right symbolic expression, we think of this as `controlling for' the regressor. The $\CA$'s loss parameters are chosen from exactly the same space as the full system, but with no RL, they are randomised at each step, and a full $50$ steps are run in each episode. 
    \\
    \item $\mathcal{M}_0$ - Full System: The system with all agents and environmental components. In $\mathcal{M}_0:\CA+\SA+$Provability, we use the usual automated prover $A$ that assigns the discrete score $\rho(s)$ (Eq. \ref{eq:discrete_proof_metric}). In $\mathcal{M}_0+$Noise : $\CA+\SA+$Noisy Provability, we use a stochastic variant, adding Gaussian noise to $\rho(s)$ and reducing the provability threshold. This is a basic test for how the system performs in a more uncertain, general setting.
    \\
    \item $\mathcal{M}_1$ - $\CA+\SA$: This model has both agents and RL, so attention to the datapoints varies dynamically. However, there is no reward for provability, to assess its value. We still compute $\rho(s)$ to check for early terminations but it is not passed to the agents, who try to maximise their smaller, intermediate rewards.
    \\
    \item $\mathcal{M}_2$ - $\CA+$Provability: This uses just the Conjecturing agent and the usual provability score, whilst fixing the weights $\lambda_i=1$ for all datapoints, to study the impact of the Skeptical agent.
\end{itemize}
\subsection{$\mathcal{D}_0$ results}
\label{sec:D0results}
We first consider experiments with $\mathcal{D}_0$, the situation most similar to Euler's own. Note that the percentage of statements with $b_1$ is the relevant metric.

In this case, \textbf{the full system completes Learning problem \ref{lp1}}, and \textbf{is the only model to do so}. We believe this is clearly a non-trivial use of the concepts and provides a partial fulfilment of the challenge outlined in §\ref{sec1.5}.

\begin{table}[!htbp]
\centering
\large 
\setlength{\tabcolsep}{3pt} 
\renewcommand{\arraystretch}{1}

\sbox{\tablebox}{%
    \resizebox{\linewidth}{!}{%
        \begin{tabular}{lcccc}
        \toprule
        Model
        & \begin{tabular}[c]{@{}c@{}} Unique atomic formulae\\ (Total statements) \end{tabular}
        & $\chi\;\%$
        & $b_1\;\%$
        & \begin{tabular}[c]{@{}c@{}} \% Proven\\ w/ $\chi$ or $b_1$\end{tabular} \\
        \midrule
        
        Only $\CA$ &
        723 (750)
        & 0
        & 0 
        & 0 \\
        \midrule
        
        $\mathcal{M}_0$ &
        522 (410)
        & \textbf{2.47} [1.07, 4.1]
        & \textbf{5.67} [4.5, 7.22]
        & \textbf{12.72} [9.83, 16.83] \\ 
        \midrule

        $\mathcal{M}_0$ + Noise & 760 (750)
        & 1.15 [0.51, 1.78]
        & 1.99 [1.19, 2.93] 
        & 7.84 [1.85, 16.22] \\ 
        \midrule
        
        $\mathcal{M}_1$ 
        & 861 (750) 
        & 2.44 [1.62, 3.34]
        & 3.24 [2.28, 4.26]
        & 7.02 [5.23, 8.86] \\ 
        \midrule
        
        $\mathcal{M}_2$ & 278 (389)
        & 2.39 [0.81, 4.21]
        & 4.63 [3.07, 6.18] 
        & 5.90 [3.52, 8.87] \\ 

        \bottomrule
        \end{tabular}%
    }%
}

\makebox[\textwidth][c]{%
    \begin{minipage}{\wd\tablebox}
        \usebox{\tablebox}
        
        \vspace{0.5em}
        \caption{ Comparison of all models on $\mathcal{D}_0$. The second column shows the number of different atomic formulae and the total number of statements produced during evaluation. In columns three through five, entries in bold show the model that performed the best with respect to the metric. The proportions of statements with $\chi$, $b_1$ and proven statements with $\chi$ or $b_1$ are shown as 95\% confidence intervals, computed via `cluster bootstrap' over evaluation episodes (10{,}000 resamples). Each resample contains a randomized selection of episodes; we report the range containing 95\% of the observed percentages. Although each episode is independent, within a given episode the statements are highly related, conditioned on the policies learned by the agents during training.
        }
        \label{tbl:initial_numbers}
    \end{minipage}
}
\end{table}
To better understand this, we examine the performance of each model in noticing and using $\chi$ and $b_1$ in Table \ref{tbl:initial_numbers}. This shows several patterns. Firstly, the `Only $\CA$' model \textbf{never} spots $\chi$ or $b_1$, strongly suggesting that the exploration of $\mathcal{S}$ is infeasible without deliberate strategies. Despite using the exact same loss as all other models (drawing from the same space of priors), symbolic regression on its own entirely fails at the learning task.

With the introduction of RL, models $\mathcal{M}_0-\mathcal{M}_2$ produce $\chi$ and $b_1$ at small but non-trivial rates. We argue that these concepts are helping the agents to achieve their goals, without direct encouragement. This is most true for the full system, with the highest rates in all categories. In the second column, models with the $\SA$ explore many more unique blocks per statement, as a dynamic data distribution improves the efficiency of exploration. 
 
\par Although finding $\chi$ and $b_1$ is evidently hard, the superior performance of $\mathcal{M}_0$ suggests the importance of the full approach. To properly test this we present an ablation study in Figure \ref{fig:ablation_D0}, showing the \emph{relative} metrics between models. Note that we are now distinguishing the ability to notice and use $\chi$ and $b_i$ individually.

\begin{figure}[!htbp]
    \centering
    \makebox[\textwidth][c]{%
        \begin{minipage}{\linewidth}            \includegraphics[width=\linewidth]{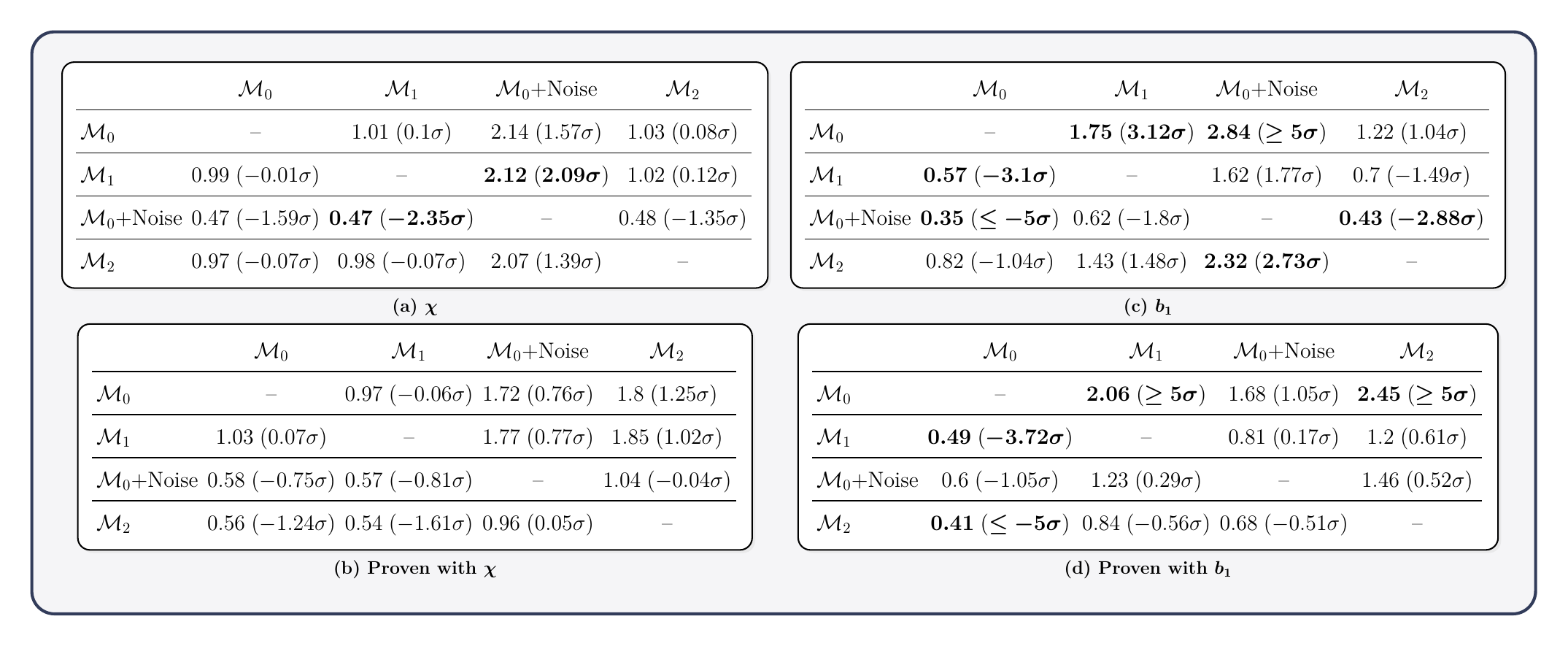}
            \caption{ These tables show the pairwise performance ratios, $i/j$, for models in row i and column j on the dataset $\mathcal{D}_0$, with respect to all metrics. Again, we use cluster bootstrap (10{,}000 resamples). We compute a one-sided effective $\sigma$ by mapping the bootstrapped probability that model $i$ outperforms model $j$ to the corresponding normal distribution. This gives an interpretable measure of the confidence in the direction of the effect, standard practice when the underlying statistic is not Gaussian \cite{cowanAsymptoticFormulaeLikelihoodbased2011}.
            Entries in bold mean that the ratio is empirically significant (i.e. $\geq2\sigma$), but in many cases the effect is stronger. $\geq5\sigma$ ($\leq5\sigma$ resp.) means that the effect is at least $5\sigma$ under our resampling resolution. 
            }
            \label{fig:ablation_D0}
        \end{minipage}%
    }
\end{figure}

It is now clear that the full system significantly outperforms its ablated counterparts. There is a large difference in the rates of noticing the two features, with $\mathcal{M}_0$ using $b_1$ much more than $\chi$. This is an intriguing parallel with the history of §\ref{sec1.4} and §\ref{sec1.5}, supporting our belief that $b_1$ is much more difficult to notice through naive pattern detection. Historically, this was a proof-formed concept, developed to explain the variation of $\chi$. A similar story emerging in our experiments suggests that the system genuinely captures feedback between questioning and answering. 

The model with noisy provability underperforms others with another missing component ($\mathcal{M}_1$ in finding $\chi$ and $\mathcal{M}_2$ in finding $b_1$). Clearly, the choice of $\rho$ has a meaningful effect. But recall that in Table \ref{tbl:initial_numbers}, $\mathcal{M}_0$ + Noise is still successful at noticing $\chi$ and $b_1$, and is incomparably better than regression alone. 

\subsection{Varying $\mathcal{D}_i$ results}\label{sec:results_varying_data}
For our final experiments with our system, we vary the dataset over all the options of §\ref{sec:2.1}.

Most importantly, in the case of $\mathcal{D}_2$ \textbf{the full system completes Learning problem \ref{lp1}}, producing a statement containing all of the relevant concepts.
Since the premise set here only contains the rank-nullity theorems, we believe that this essentially fulfils the open challenge described in §\ref{sec1.5}. 

To better understand the relationship between data and premises and its impact on learning, we study $(i)$ $\mathcal{M}_0$ vs $\mathcal{M}_1$ on $\mathcal{D}_2$ and $(ii)$ how $\mathcal{M}_0$ performs on varying $\mathcal{D}_i$. 

\begin{figure}[!htbp]
    \centering
    \makebox[\textwidth][c]{\begin{minipage}{\linewidth}
    \includegraphics[width=\linewidth, center]{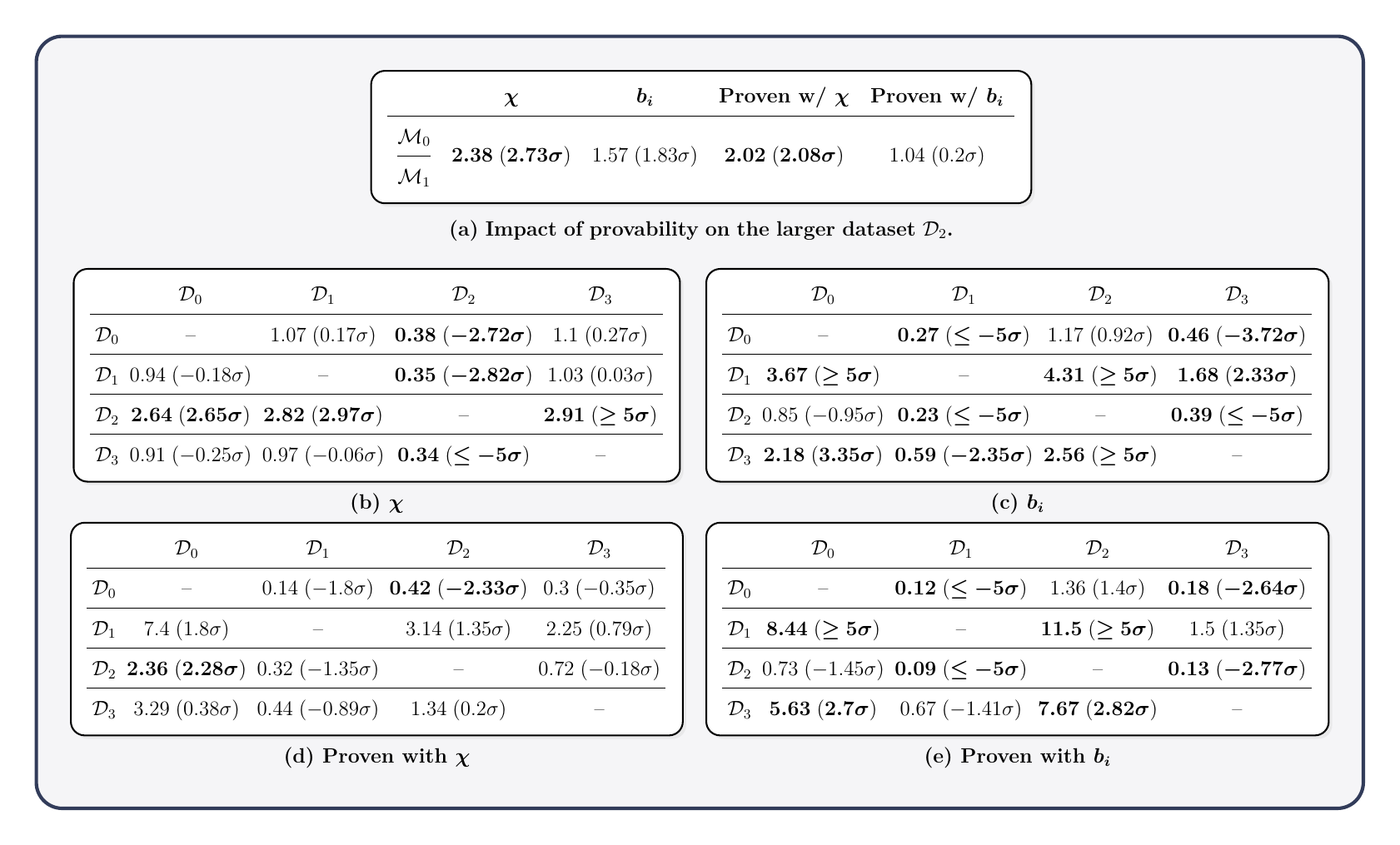}
    \caption{Table (a) gives the performance ratios of $\mathcal{M}_0$ vs $\mathcal{M}_1$ on dataset $\mathcal{D}_2$. 
    The tables in (b) - (e) again show the pairwise performance ratios, $i/j$ for row i and column j, but we are now varying the data and premises rather than the model, and each entry uses the full system (i.e. $\CA+\SA+$Provability). These ablations capture the changing nature of the learning problem as more diverse triangulations are added. We again use cluster bootstrap for sampling (with 10{,}000 resamples) and compute a one-sided effective $\sigma$, with $\sigma \geq2$ being empirically significant. Finally, $\geq5\sigma$ ($\leq5\sigma$ resp.) still indicates that the effect is at least $5\sigma$ under our resampling resolution.}
    \label{fig:vary_data}
    \end{minipage}
    }
\end{figure}

Firstly, Figure \ref{fig:vary_data}(a) shows that provability still improves performance in most cases. Secondly, subfigures \ref{fig:vary_data}(b) - (e) contain several patterns. Introducing the Klein bottle (in $\mathcal{D}_1$) leads to a very large improvement in noticing Betti numbers, compared to $\mathcal{D}_0$ and $\mathcal{D}_2$.  This is probably because the system can notice $b_2$ without $b_0$ also changing\footnote{Relative to spheres and tori.}, as happens in the disconnected $\mathcal{D}_2$ case. We see this further supported by the absence of a comparable improvement in $\mathcal{D}_2$ vs $\mathcal{D}_0$. Lastly, although the largest dataset $\mathcal{D}_3$ still outperforms $\mathcal{D}_0$ on some metrics, the slightly smaller effect size is perhaps due to its now significant diversity and size. Taken together, the final results strongly suggest a link between the mathematical structure encoded in the data and the performance of the system.
 
\par 
The statements that completed the Learning problem are included in Appendix \ref{app:statments}, with a collection of others. These use both definitions of $\chi$, and sometimes the Betti numbers alone, in a variety of suggestive ways. The Appendix also contains statements found by the system that do not complete the Learning problem, for example, trivialities or expressions involving the Betti numbers that are false. 
\par A sensitivity analysis shows that the rates of noticing $\chi$ and homology for a range of hyperparameters are in the confidence bands of the results here (see  Appendix \ref{app:statments}).
\newpage
\subsection{LLM results} \label{llmsec}
The results of §\ref{sec:2.1} - \ref{sec:results_varying_data} suggest that Thesis \ref{thesis1} may be important for a more general approach to AI in mathematics. The ablations demonstrated the significance of a relationship between inductive and deductive reasoning and helped to control for the difficulty of the learning problem. However, it would be useful to test how hard the task is for modern reasoning models. 

The past few months have seen countless demonstrations of the remarkable capabilities of these systems in producing solutions to major open problems \cite{openaiMAth}. Yet if we accept the first part of Thesis \ref{thesis1}, it is also important to see how they perform in reasoning tasks that would traditionally be classed as inductive. Part of the point is that inductive and deductive reasoning cannot be clearly separated. Nevertheless, recent events have mostly shown their proficiency in deductive tasks, even if they use a broader range of reasoning to achieve them. We are not aware of any existing systematic investigation of LLM autonomous mathematical pattern-spotting in the literature\footnote{Our tests here are not systematic either, but provide some insight whilst also contextualising our system's results.}.

There is an issue with testing commercially available models on Learning Problem \ref{lp1}: they have all been extensively trained on the concepts of homology, Betti numbers, and the Euler characteristic. They have even seen the historical story of \ref{sec1.4} many times. In our view, for a mathematician already comfortable with basic topology, such a pattern-spotting exercise would be fairly routine.

Nevertheless,  we consider:
\begin{learningproblem}
\normalsize
Can an LLM recognise the concept of homology from polyhedral data?
\label{lp2}
\end{learningproblem}

Concretely, we pass the simplest dataset, $\mathcal{D}_0$ containing spheres and tori, to a group of LLMs, alongside prompts of varying suggestiveness. The polyhedral dataset was given as text in the prompt or attached to it (see Appendix \ref{app:LLM_baseline} for further details). To judge success, we read the output and assessed whether the LLM had noticed and interpreted the relevant patterns. We found that it was always sufficiently clear whether or not the model had `got' the right idea, i.e. that the data encodes a polyhedral chain complex along with some references to homology via Euler characteristic and/or Betti numbers.
\begin{table}[!htbp]
\centering
\large 
\setlength{\tabcolsep}{3pt} 
\renewcommand{\arraystretch}{1}

\sbox{\tablebox}{%
    \resizebox{0.9\linewidth}{!}{%
        \begin{tabular}{cccccc}
        \toprule
        Prompts
        & Opus 4.6 
        & Opus 4.8
        & Fable 5.2
        & GPT-5.2 Pro
        & GPT-6.2-Astra \\
        \midrule
        
        p1--p6 
        & 0/60
        & 0/60
        & 12/60
        & 0/60
        & 0/60\\
        \midrule
        
        p7
        & 3/10
        & 1/10
        & 10/10
        & 0/10
        & 1/10 \\

        \bottomrule
        \end{tabular}%
    }%
}

\makebox[\textwidth][c]{%
    \begin{minipage}{\wd\tablebox}
        \usebox{\tablebox}
        
        \vspace{0.5em}
        \caption{
Summary of the performance of the models in the learning tasks. Entries are the proportions of successful responses of the model of that column to the prompt (or prompt collection) corresponding to a row. Prompts $p1$ through $p6$ are grouped together and differ from each other by the adding of phrases and small wording variations; $p7$ has its own row since it contains a very explicit hint towards homology, suggesting a search for: ‘...invariants of shape encoded through the algebraic relations...’. 
}
        \label{tbl:llm_baselines_main}
    \end{minipage}
}
\end{table}

Table \ref{tbl:llm_baselines_main} summarises the results. Of the LLMs tested, \textbf{no Anthropic model released before Fable $5.2$} and \textbf{no GPT model at all (including $6.2$-Astra) was able to successfully complete the Learning Problem \ref{lp2},} for prompts $p1-p6$. This is despite the prompts often containing a significant amount of helpful suggestions, for example $p3$: \bigskip

\emph{Here is a mathematical dataset in which each row corresponds to one datapoint. By analyzing patterns in the data, produce a well-formed first order formula. It might be helpful to consider groupings within a given row as well as linear maps. Try to produce an expression that you find mathematically interesting and that is true on a substantial portion of the data. Explain any notation that is not standard.}
\\

This yields the following response from Opus $4.8$:
\\

\noindent\fbox{\parbox{\dimexpr\linewidth-2\fboxsep-2\fboxrule}{%
I'll call the eight columns $x_1,\dots,x_8$. The last column $x_8$ is constant ($=1$, the ``true'' label), so I focus on $x_1,\dots,x_7$.

\smallskip\textbf{What is rigid in every row.} Reading the columns off, the following hold on $100\%$ of the rows:
\[x_4=x_1,\qquad 2x_1=3x_2,\qquad x_5=x_3-1,\qquad x_6=x_2-1,\qquad (x_7-x_2)^2=1.\]
So writing $n=x_2/2=x_1/3$, every row looks like
\[\bigl(3n,\ 2n,\ x_3,\ 3n,\ x_3-1,\ 2n-1,\ x_7\bigr),\qquad x_3\in\{n,\ n+2\},\quad x_7\in\{2n-1,\ 2n+1\}.\]
The only genuinely ``free'' bit is one coordinated choice: either $(x_3,x_7)=(n,\,2n+1)$ or $(x_3,x_7)=(n+2,\,2n-1)$. (Indeed both occur for the same $n$: e.g.\ $n=313$ gives both the rows $939,626,\mathbf{313},939,312,625,\mathbf{627}$ and $939,626,\mathbf{315},939,314,625,\mathbf{625}$, so this is a real coin-flip, not a function of $n$.)

\smallskip\textbf{The interesting formula.} The first branch, $(x_3,x_7)=(n,2n+1)$, gives rows of the beautiful form $\bigl(3c,\,2c,\,c,\,3c,\,c-1,\,2c-1,\,2c+1\bigr)$ ($c=x_3$), and these satisfy a unimodular (determinant $=1$) identity:
\[\boxed{\;\Phi:\quad (x_4=x_1)\ \wedge\ (2x_1=3x_2)\ \wedge\ (x_5=x_3-1)\ \wedge\ (x_6=x_2-1)\ \wedge\ \bigl(x_3x_6-x_5x_7=1\bigr)\;}\]
Check of the last conjunct for $x_3=c,\ x_5=c-1,\ x_6=2c-1,\ x_7=2c+1$: $c(2c-1)-(c-1)(2c+1)=(2c^2-c)-(2c^2-c-1)=1$. Numerically, row 1: $138\cdot275-137\cdot277=37950-37949=1$; row 17: $67\cdot133-66\cdot135=8911-8910=1$.~\checkmark\ $\Phi$ is true on the whole $(x_3,x_7)=(n,2n+1)$ family, a substantial portion ($\sim45\%$) of the data, while on the other family it fails only in the last conjunct, where instead $x_3x_6-x_5x_7=x_6$ (since there $x_7=x_6$).

\smallskip\textbf{Why it's interesting.} The condition $x_3x_6-x_5x_7=1$ says the fractions $\frac{x_5}{x_6}=\frac{c-1}{2c-1}$ and $\frac{x_3}{x_7}=\frac{c}{2c+1}$ are Farey/Stern--Brocot neighbours: $\left|\begin{smallmatrix}x_3&x_5\\x_7&x_6\end{smallmatrix}\right|=1$. Both fractions sit just below $\tfrac12$ (which is $x_3/x_2=c/2c$), and the row packages the integer $c$ together with this adjacent pair of mediant-neighbours of $\tfrac12$. The non-standard notation here is only $x_i$ for ``the $i$-th comma-separated entry of a row.''}}
\\

From the perspective of human-aligned understanding, this is not satisfactory. The formula is not true on the whole dataset. More importantly, the model makes seemingly tangential connections to number theory, even after being hinted at to consider linear algebra. Parts of the formula identify meaningful relationships: $2V=3E$, connectedness, or an expression that distinguishes tori from spheres. But it is perhaps more telling that the model can spot such patterns, but is unable to draw the connection to topology. The choice to interpret in terms of a `Farey/Stern-Brocot' relation was not consistent, occurring in $7$ out of $70$ runs. Across different tests, the LLMs gave a great range of interpretations, for example, that datapoints come from `a very low dimensional variety', conform to a `Bezout identity', or that a matrix of features lies in `the modular group central to number theory, hyperbolic geometry, and the theory of modular forms'. 

Ultimately, the most direct sign here that the LLM does not possess a human-aligned notion of interestingness, is its lack of satisfactory interpretation. In addition, amongst the outputs considered failures, none, except one, produced any linear relationship involving $\chi$ and  $b_i$ (even in this case, the LLM did not interpret it as such).

For a run judged successful, Fable $5.2$ gave the following output for the same prompt:

\noindent\fbox{\parbox{\dimexpr\linewidth-2\fboxsep-2\fboxrule}{%
\textbf{Reading of the data.} Writing a row as $(x_1,\dots,x_8)$, each row looks like the linear-algebra data of the simplicial chain complex of a triangulated closed orientable surface, either a sphere or a torus:
\[C_2\xrightarrow{\ \partial_2\ }C_1\xrightarrow{\ \partial_1\ }C_0,\qquad\text{with }\partial_1\circ\partial_2=0.\]
[\dots]

\smallskip\textbf{Formula.} For every row $(x_1,\dots,x_8)$:\quad
$\forall\,\text{row}:$ $x_1=x_4$
$\wedge$ $2x_1=3x_2$~(every edge lies in two triangles, every triangle has three edges)
$\wedge$ $x_5+x_7=x_4$~(rank--nullity for $\partial_1$)
$\wedge$ $x_3-x_5=x_8$~($b_0=\dim\operatorname{coker}\partial_1=1$)
$\wedge$ $x_2-x_6=x_8$~($b_2=\dim\ker\partial_2=1$)
$\wedge$ $x_6\le x_7$~($\operatorname{im}\partial_2\subseteq\ker\partial_1$)
$\wedge$ $\exists\,g\in\{0,1\}:\ x_7-x_6=2g$~($b_1=2\cdot\text{genus}$)

\smallskip\textbf{Consequence (Euler--Poincar\'e).} Combining the conjuncts gives
\[x_3-x_1+x_2=x_8-(x_7-x_6)+x_8=2-2g,\]
that is, $V-E+F=b_0-b_1+b_2$.

[\dots]

\smallskip\textbf{Notation.} $b_0=x_3-x_5$, $b_1=x_7-x_6$ and $b_2=x_2-x_6$ are the Betti numbers (dimensions of the homology groups $H_0$, $H_1$, $H_2$ of the complex), and $g$ is the genus of the surface.}}
\\
\\
(here [...] indicate omissions for the sake of space). This excerpt clearly demonstrates a strong degree of human-aligned understanding of the data. However, for prompts $p1-p6$, even Fable $5.2$ does not consistently provide such a satisfactory interpretation, with success on only $12$ of $60$ runs.

All of the LLMs tested are `reasoning models', and even Opus $4.8$ performed very strongly on the US Olympiad \cite{anthropic2026opus48}. The main reason for testing with these older models was to use LLMs trained before the first arXiv version of this paper\footnote{Models trained afterwards were largely able to give details of the paper in a sandboxed test.}. However, it is perhaps more relevant to consider commercially available frontier models. For example, the mathematical problem-solving ability of Astra is extremely impressive in a research setting \cite{EpochFrontiermathTier4V22026}. Yet given a total of sixty prompts $p1-p6$, it was unable to produce a successful response. For prompt $p7$:
\\
\\
\emph{Here is a mathematical dataset in which each row corresponds to one datapoint. By analyzing patterns in the data, produce a well-formed first order formula. It might be helpful to think of invariants of shape encoded in the algebraic relations of the columns. Try to produce an expression that you find mathematically interesting and that is true on a substantial portion of the data. Explain any notation that is not standard,}
\\
\\
it only produced one successful output out of ten runs, despite being urged to consider `\textbf{invariants of shape encoded in... algebraic relations}' (in contrast, Fable $5.2$ gave ten out of ten successful outputs to this highly suggestive hint). A typical response from Astra to $p7$ was:
\\

\noindent\fbox{\parbox{\dimexpr\linewidth-2\fboxsep-2\fboxrule}{%
The rows lie on two parallel affine lines. Writing $x_1,\dots,x_8$ for the entries of a row, the following first-order formula over the integers holds for every supplied row:
\[\begin{aligned}
\exists t\,\exists s\;\Bigl[\,&(s=-1\lor s=1)\ \land{}\\
&(x_1,\ldots,x_8)=(3t,\,2t,\,t+s+1,\,3t,\,t+s,\,2t-1,\,2t-s,\,1)\,\Bigr].
\end{aligned}\]
Tuple equality abbreviates the conjunction of the eight coordinate equalities.

\smallskip To see the pattern, set $t=x_1-x_2$. The first two entries are $3t$ and $2t$, while the fifth and seventh are $t$ and $2t$ shifted by one unit in opposite directions. This gives the invariant $x_5+x_7=x_1$ on both lines. The two rows beginning with $939,626$ illustrate why both signs matter: they have the same $t=313$, but their fifth and seventh entries are $(312,627)$ and $(314,625)$, respectively.}}
 \\
 \\
 Until recent events, Astra was often considered the most mathematically competent model; its failure on even the very helpful $p7$ supports Thesis \ref{thesis1}'s more expansive concept of mathematical intelligence. In §\ref{sec:D0results}-\ref{sec:results_varying_data}, simple proof feedback gave interpretable outcomes within a system that is much less accomplished in problem-solving; in contrast, we rarely observed that the LLMs justified the patterns they spotted with formal deductive reasoning. 

 This section does not claim that LLMs cannot be optimised to perform better at a task that would be fairly routine for humans mathematicians with some knowledge of topology (especially in the $p7$ case). However, if we accept Thesis \ref{thesis1}, it is informative to consider the current landscape of their capabilities. This landscape may be more jagged than is presently believed, especially from the perspective of human understanding.

\section{Discussion}\label{sec:Discussion}
The main theme of this paper has been the inseparable links between conjecturing, proving, and concept-formation in the creation of ideas with mathematical value. In §\ref{refdrawbacks} we argued that AI systems will struggle to create their \emph{own} interesting results without accommodating this in some way. Their ability to solve set problems will likely continue to improve, producing more striking outcomes \cite{openaiMAth}.
But if we accept Thesis \ref{thesis1} and seek general machine mathematical intelligence, there are different challenges. Our LLM experiments suggest that great skill in problem solving does not necessarily correlate with all aspects of mathematical intelligence, at least at present. 

In this work, we presented a new framework for computation-driven mathematical discovery, based on the evolution of conjectures and data through proof feedback. The benchmark was Learning problem \ref{lp1}, thought of as `rediscovering' the concept of homology.  This required an AI to detect and meaningfully combine the Betti numbers and $\chi$. Our full system completed this task. 

In §\ref{refdrawbacks}, we discussed the difficulties associated with defining mathematical `interestingness' through a formula. Our results suggest that human-aligned `interestingness' may not need to be directly optimised for; it can emerge through a dialectic between agents pursuing individual goals. Furthermore, the aim was not to directly study the proving aspects of mathematical practice. Since we placed the prover in the environment, no agent was required to navigate the large space of all proofs, and even infrequent or corrupted proofs had a large impact on alignment.

The history of Euler's conjecture was a natural test for a mathematical AI. However, the realisation that simplicial complexes and the algebra of chain vector spaces encode topology was a breakthrough\footnote{The first mathematician to strongly emphasise the algebraic perspective was probably Noether \cite{hiltonBriefSubjectiveHistory1988}.}, perhaps hard to entirely control for in our experiments. We consider possible objections, which qualify the idea of `rediscovery'. 

Our dataset was essentially a large collection of diverse polyhedra; this was necessary for using modern statistical tools. We believe that presenting these as incidence matrices, preprocessed with standard functions, was the most natural way to present the problem today. Secondly, although the translation `interprets' these as linear maps, the system may only know this through the very coarse-grained proof score. This level of `knowledge' is far below that of the LLMs in §\ref{llmsec}, which struggled to recognise an example of a concept upon which they have been trained extensively. Lastly, Learning Problem \ref{lp1} involved simplicial homology, not other types. But this was the first and crucial definition.

Past work suggests that mathematically significant objects may lie in very small portions of parameter space \cite{shehper_what_2024}. In §\ref{sec:D0results} the intrinsic difficulty of detecting $\chi$ and $b_i$ was demonstrated by the `regression only' trial, which did not produce them. This used the same regressor (as the full system) which, though carefully designed, we needed to control for. The regression's failure contextualises the structure built into the symbolic vocabulary.

Important ideas can seem obvious in hindsight. The Euler characteristic - `the paradigmatic local-global principle in mathematics' - sharpens the intuitive idea of a hole by counting `building blocks' \cite{Harris_2024}. In geometry, the proof of the Gauss-Bonnet theorem, relating local curvature to global topology, can be thought of as a continuous limit \cite{bott_differential_2008}. Euler's ideas have now developed into the abstract Index theorems, which cover Riemann-Roch as a special case \cite{gilkeyInvarianceTheoryHeat2018}. 

The ablations make a strong argument for the full dynamic. Spotting $\chi$ and $b_i$ helped produce long and provable statements. The evolving attention to data encouraged efficient exploration. Valuable notions emerged through the interplay of local optimisation, a suggestive computational analogue to §\ref{sec1.3}. It is important that the ablations showing this stand on their own, assuming choices made in the environment and agents. 

Details of the system could be changed whilst maintaining the overall flow. It is believed that the sequential ordering of data points impacts learning in ML \cite{polu_formal_2023}, focused on increasing difficulty. The Skeptical agent uses a similar idea to encourage constrained exploration, but could be implemented differently. The regression was adapted from a package designed for natural science and was still not optimal for mathematical expressions \cite{cranmerInterpretableMachineLearning2023}.
Finally, we used a fairly rudimentary prover and direct translation \cite{song_lean_2025}. Using auto-formalisation, it is plausible that the system would try linear algebra arguments independently.

Given the nature of our collaboration, a proof of principle was a natural first step. But we hope that a version of our dynamic could be applied in research, perhaps by using a revised or enriched notion of `provability'. Highly sophisticated provers are now accessible and could be used in the environment, with hand-checked definitions and theorems containing background theoretical structure. Such a modified system could instantiate the framework of this paper with more complex proofs\footnote{An experiment of this nature in the topology of three-manifolds is the subject of ongoing work joint between the first two authors and Djordje Mihajlovic.}. The history of Euler's conjecture shows that deep conceptual structure can have roots in activities as concrete as the counting of faces, vertices and edges. Our results demonstrate that even the most straightforward proof feedback\footnote{Just a yes or no score.} can transform outcomes within the right dialectic.

In a 1947 lecture Turing made many predictions about computers\footnote{Lecture to the London Mathematical Society, 20 February 1947.} \cite{Turing1947}. He anticipated that they would develop unpredictable behaviour, `learn(ing) through experience' to become good at games like chess. In mathematics, they would do `more and more themselves', such as the `manipulations of mathematical formulae'. However, truly intelligent behaviour would require integration with existing knowledge and judgments of value. Speaking specifically about mathematics applications, he argued that `the machine must be allowed to have contact with human beings in order that it may adapt itself to their standards.' We interpret him as meaning that what makes a good machine mathematician is much the same as what makes a good human mathematician; that was the basis of this paper.

\noindent\textbf{Acknowledgements:} The authors thank Minhyong Kim for many helpful discussions. They also thank the anonymous referees for their careful reading and insightful comments, especially their suggestion to consider external baseline experiments. Daattavya Aggarwal was supported by the UKRI grant G111021. He thanks Anand Rao Tadipatri for suggestions on using Lean. Oisin Kim was supported by the Accelerate Programme for Scientific Discovery.

\medskip
\noindent Code Availability: The full code, trained models and data, as well as all prompts and responses from the LLMs is available on the repository: \url{https://github.com/daattavya98/MathWorld-learning-homology}.

\bibliographystyle{unsrtnat}
\bibliography{sample}

\begin{thebibliography}{66}
\providecommand{\natexlab}[1]{#1}
\providecommand{\url}[1]{\texttt{#1}}
\expandafter\ifx\csname urlstyle\endcsname\relax
  \providecommand{\doi}[1]{doi: #1}\else
  \providecommand{\doi}{doi: \begingroup \urlstyle{rm}\Url}\fi

\bibitem[Lakatos(1976)]{Lakatos_Worrall_Zahar_1976}
Imre Lakatos.
\newblock \emph{Proofs and Refutations: The Logic of Mathematical Discovery}.
\newblock Cambridge University Press, Cambridge, 1976.

\bibitem[Grothendieck(2022)]{grothendieckRecoltesSemaillesReflexions2021}
Alexander Grothendieck.
\newblock \emph{{R\'ecoltes et semailles: r\'eflexions et t\'emoignage sur un pass\'e de math\'ematicien}}.
\newblock Tel. {\'E}ditions Gallimard, Paris, 2022.
\newblock ISBN 9782072889752.
\newblock English translation by Tong Zhou: \url{https://tongchow.github.io/ReS.html}.

\bibitem[Hilbert(1902)]{hilbert1902mathematical}
David Hilbert.
\newblock Mathematical problems.
\newblock \emph{Bulletin of the American Mathematical Society}, 8\penalty0 (10):\penalty0 437--479, July 1902.

\bibitem[Harris(2024)]{Harris_2024}
Michael Harris.
\newblock Is alphageometry a dead end?, Jan 2024.
\newblock \url{https://siliconreckoner.substack.com/p/is-alphageometry-a-dead-end}.

\bibitem[Cromwell(1997)]{cromwell1997polyhedra}
Peter~R. Cromwell.
\newblock \emph{Polyhedra}.
\newblock Cambridge University Press, Cambridge, 1997.
\newblock ISBN 978-0-521-55432-9.

\bibitem[Richeson(2008)]{richeson_eulers_2008}
David~Scott Richeson.
\newblock \emph{Euler's gem: the polyhedron formula and the birth of topology}.
\newblock Princeton university press, Princeton, 2008.
\newblock ISBN 978-0-691-12677-7.

\bibitem[Polya(1954)]{polya1}
George Polya.
\newblock \emph{Mathematics and Plausible Reasoning, Volume 1: Induction and Analogy in Mathematics}.
\newblock Princeton University Press, Princeton, 1954.
\newblock ISBN 9780691080055.
\newblock URL \url{http://www.jstor.org/stable/j.ctv14164db}.

\bibitem[Alama(2009)]{alama}
Jesse Alama.
\newblock \emph{Formal Proofs and Refutations}.
\newblock Doctoral dissertation, Stanford University, 2009.
\newblock \url {https://www.proquest.com/openview/edb2c4352c2d778f7a421c6a357f080d/1?pq-origsite=gscholar&cbl=18750}.

\bibitem[Carlsson(2009)]{carlsson_topology_2009}
Gunnar Carlsson.
\newblock Topology and data.
\newblock \emph{Bulletin of the American Mathematical Society}, 46\penalty0 (2):\penalty0 255--308, January 2009.
\newblock ISSN 0273-0979.
\newblock \doi{10.1090/S0273-0979-09-01249-X}.
\newblock URL \url{http://www.ams.org/journal-getitem?pii=S0273-0979-09-01249-X}.

\bibitem[Green et~al.(2012)Green, Witten, and Schwarz]{green_superstring_2012}
Michael~B. Green, Edward Witten, and John~H. Schwarz.
\newblock \emph{Superstring theory: {Volume} 1: {Introduction}}.
\newblock Cambridge monographs on mathematical physics. Cambridge University Press, Cambridge, 25th anniversary edition edition, 2012.
\newblock ISBN 978-1-139-24856-3 978-1-139-53120-7 978-1-139-52653-1.

\bibitem[Poincaré(2010)]{poincare_papers_2010}
Henri Poincaré.
\newblock \emph{Papers on {Topology}}, volume~37 of \emph{History of {Mathematics}}.
\newblock American Mathematical Society, Providence, Rhode Island, September 2010.
\newblock ISBN 978-0-8218-5234-7 978-1-4704-1840-3.
\newblock \doi{10.1090/hmath/037}.

\bibitem[Chemla(2015)]{chemlaHistoryMathematicalProof2015}
Karine Chemla.
\newblock \emph{The history of mathematical proof in ancient traditions}.
\newblock Cambridge university press, Cambridge, 2015.
\newblock ISBN 978-1-107-52753-9.

\bibitem[Hales(2014)]{Hales_2014}
Thomas~C. Hales.
\newblock \emph{Mathematics in the age of the Turing machine}, page 253–298.
\newblock Lecture Notes in Logic. Cambridge University Press, Cambridge, 2014.

\bibitem[Appel and Haken(1977)]{Appel1977429}
K.~Appel and W.~Haken.
\newblock Every planar map is four colorable part {I}: {Discharging1}.
\newblock \emph{Illinois Journal of Mathematics}, 21\penalty0 (3):\penalty0 429 -- 490, 1977.
\newblock \doi{10.1215/ijm/1256049011}.
\newblock URL \url{https://www.scopus.com/inward/record.uri?eid=2-s2.0-84972498411&doi=10.1215%2fijm%2f1256049011&partnerID=40&md5=8bc9927e93f9d28b3be111d15703aebf}.

\bibitem[Appel et~al.(1977)Appel, Haken, and Koch]{Appel1977491}
K.~Appel, W.~Haken, and J.~Koch.
\newblock Every planar map is four colorable part {II}: {Reducibility1}.
\newblock \emph{Illinois Journal of Mathematics}, 21\penalty0 (3):\penalty0 491 -- 567, 1977.
\newblock \doi{10.1215/ijm/1256049012}.
\newblock URL \url{https://www.scopus.com/inward/record.uri?eid=2-s2.0-84972500815&doi=10.1215%2fijm%2f1256049012&partnerID=40&md5=b1c41a507c61a005a58bc2725832857d}.

\bibitem[Commelin et~al.(2022)Commelin, Topaz, et~al.]{lean-liquid}
Johan Commelin, Adam Topaz, et~al.
\newblock Liquid tensor experiment.
\newblock \url{https://github.com/leanprover-community/lean-liquid}, 2022.

\bibitem[{OpenAI}(2025)]{openai_gpt5_math_discovery_2025}
{OpenAI}.
\newblock How gpt-5 helped mathematician ernest ryu solve a 40-year-old open problem, November 2025.
\newblock \url {https://openai.com/index/gpt-5-mathematical-discovery/}.

\bibitem[{Romera-Paredes} et~al.(2024){Romera-Paredes}, Barekatain, Novikov, Balog, Kumar, Dupont, Ruiz, Ellenberg, Wang, Fawzi, Kohli, and Fawzi]{romera-paredesMathematicalDiscoveriesProgram2024}
Bernardino {Romera-Paredes}, Mohammadamin Barekatain, Alexander Novikov, Matej Balog, M.~Pawan Kumar, Emilien Dupont, Francisco J.~R. Ruiz, Jordan~S. Ellenberg, Pengming Wang, Omar Fawzi, Pushmeet Kohli, and Alhussein Fawzi.
\newblock Mathematical discoveries from program search with large language models.
\newblock \emph{Nature}, 625\penalty0 (7995):\penalty0 468--475, January 2024.
\newblock ISSN 1476-4687.
\newblock \doi{10.1038/s41586-023-06924-6}.

\bibitem[Novikov et~al.(2025)Novikov, Vũ, Eisenberger, Dupont, Huang, Wagner, Shirobokov, Kozlovskii, Ruiz, Mehrabian, Kumar, See, Chaudhuri, Holland, Davies, Nowozin, Kohli, and Balog]{novikov_alphaevolve_nodate}
Alexander Novikov, Ngân Vũ, Marvin Eisenberger, Emilien Dupont, Po-Sen Huang, Adam~Zsolt Wagner, Sergey Shirobokov, Borislav Kozlovskii, Francisco J.~R. Ruiz, Abbas Mehrabian, M.~Pawan Kumar, Abigail See, Swarat Chaudhuri, George Holland, Alex Davies, Sebastian Nowozin, Pushmeet Kohli, and Matej Balog.
\newblock Alphaevolve: A coding agent for scientific and algorithmic discovery, 2025.
\newblock URL \url{https://arxiv.org/abs/2506.13131}.

\bibitem[Avigad(2021)]{avigad_varieties_2021}
Jeremy Avigad.
\newblock Varieties of mathematical understanding.
\newblock \emph{Bulletin of the American Mathematical Society}, 59\penalty0 (1):\penalty0 99--117, February 2021.
\newblock ISSN 0273-0979, 1088-9485.
\newblock \doi{10.1090/bull/1726}.
\newblock URL \url{https://www.ams.org/bull/2022-59-01/S0273-0979-2021-01726-5/}.

\bibitem[Venkatesh(2024)]{venkatesh_thoughts_2024}
Akshay Venkatesh.
\newblock Some thoughts on automation and mathematical research.
\newblock \emph{Bulletin of the American Mathematical Society}, 61\penalty0 (2):\penalty0 203--210, February 2024.
\newblock ISSN 0273-0979, 1088-9485.
\newblock \doi{10.1090/bull/1834}.
\newblock URL \url{https://www.ams.org/bull/2024-61-02/S0273-0979-2024-01834-5/}.

\bibitem[Neumann(1964)]{VonNeumann1964-VONTFF}
Johann~Von Neumann.
\newblock The formalist foundations of mathematics.
\newblock In P.~Benacerraf~H. Putnam, editor, \emph{Philosophy of Mathematics}. Prentice-Hall, New Jersey, US, 1964.

\bibitem[Silver et~al.(2018)Silver, Hubert, Schrittwieser, Antonoglou, Lai, Guez, Lanctot, Sifre, Kumaran, Graepel, Lillicrap, Simonyan, and Hassabis]{silver_general_2018}
David Silver, Thomas Hubert, Julian Schrittwieser, Ioannis Antonoglou, Matthew Lai, Arthur Guez, Marc Lanctot, Laurent Sifre, Dharshan Kumaran, Thore Graepel, Timothy Lillicrap, Karen Simonyan, and Demis Hassabis.
\newblock A general reinforcement learning algorithm that masters chess, shogi, and {Go} through self-play.
\newblock \emph{Science}, 362\penalty0 (6419):\penalty0 1140--1144, December 2018.
\newblock \doi{10.1126/science.aar6404}.
\newblock URL \url{https://www.science.org/doi/10.1126/science.aar6404}.

\bibitem[Deepmind(2024)]{alphaproof}
Deepmind.
\newblock Ai achieves silver-medal standard solving international mathematical olympiad problems.
\newblock Technical report, Google, 2024.
\newblock \url{https://deepmind.google/blog/ai-solves-imo-problems-at-silver-medal-level/}.

\bibitem[Hubert et~al.(2025)Hubert, Mehta, Sartran, Horv{\'a}th, {\v Z}u{\v z}i{\'c}, Wieser, Huang, Schrittwieser, Schroecker, Masoom, Bertolli, Zahavy, Mandhane, Yung, Beloshapka, Ibarz, Veeriah, Yu, Nash, Lezeau, Mercuri, S{\"o}nne, Mehta, Davies, Zheng, Pedregosa, Li, Samuel, Velingker, Schmitt, Lockhart, Hughes, Michalewski, Sonnerat, Hassabis, Kohli, and Silver]{hubertOlympiadlevelFormalMathematical2025}
Thomas Hubert, Rishi Mehta, Laurent Sartran, Mikl{\'o}s~Z. Horv{\'a}th, Goran {\v Z}u{\v z}i{\'c}, Eric Wieser, Aja Huang, Julian Schrittwieser, Yannick Schroecker, Hussain Masoom, Ottavia Bertolli, Tom Zahavy, Amol Mandhane, Jessica Yung, Iuliya Beloshapka, Borja Ibarz, Vivek Veeriah, Lei Yu, Oliver Nash, Paul Lezeau, Salvatore Mercuri, Calle S{\"o}nne, Bhavik Mehta, Alex Davies, Daniel Zheng, Fabian Pedregosa, Yin Li, Samuel, Ameya Velingker, Simon Schmitt, Edward Lockhart, Edward Hughes, Henryk Michalewski, Nicolas Sonnerat, Demis Hassabis, Pushmeet Kohli, and David Silver.
\newblock Olympiad-level formal mathematical reasoning with reinforcement learning.
\newblock \emph{Nature}, November 2025.
\newblock ISSN 1476-4687.
\newblock \doi{10.1038/s41586-025-09833-y}.

\bibitem[Dutilh~Novaes(2020)]{dutilh_novaes_dialogical_2020}
Catarina Dutilh~Novaes.
\newblock \emph{The {Dialogical} {Roots} of {Deduction}: {Historical}, {Cognitive}, and {Philosophical} {Perspectives} on {Reasoning}}.
\newblock Cambridge University Press, Cambridge, 1 edition, December 2020.
\newblock ISBN 978-1-108-80079-2 978-1-108-47988-2 978-1-108-79092-5.
\newblock \doi{10.1017/9781108800792}.

\bibitem[Thurston(1994)]{thurstonProofProgressMathematics1994}
William~P. Thurston.
\newblock On proof and progress in mathematics, April 1994.

\bibitem[Davies et~al.(2021)Davies, Veličković, Buesing, Blackwell, Zheng, Tomašev, Tanburn, Battaglia, Blundell, Juhász, Lackenby, Williamson, Hassabis, and Kohli]{davies_advancing_2021}
Alex Davies, Petar Veličković, Lars Buesing, Sam Blackwell, Daniel Zheng, Nenad Tomašev, Richard Tanburn, Peter Battaglia, Charles Blundell, András Juhász, Marc Lackenby, Geordie Williamson, Demis Hassabis, and Pushmeet Kohli.
\newblock Advancing mathematics by guiding human intuition with {AI}.
\newblock \emph{Nature}, 600\penalty0 (7887):\penalty0 70--74, December 2021.
\newblock ISSN 0028-0836, 1476-4687.
\newblock \doi{10.1038/s41586-021-04086-x}.
\newblock URL \url{https://www.nature.com/articles/s41586-021-04086-x}.

\bibitem[Mishra et~al.(2023)Mishra, Moulik, and Sarkar]{mishra_mathematical_2023}
Challenger Mishra, Subhayan~Roy Moulik, and Rahul Sarkar.
\newblock Mathematical conjecture generation using machine intelligence, June 2023.
\newblock arXiv:2306.07277.

\bibitem[Bengio and Malkin(2024)]{bengio_machine_2024}
Yoshua Bengio and Nikolay Malkin.
\newblock Machine learning and information theory concepts towards an {AI} {Mathematician}, March 2024.

\bibitem[Shehper et~al.(2024)Shehper, Medina-Mardones, Lewandowski, Gruen, Kucharski, and Gukov]{shehper_what_2024}
Ali Shehper, Anibal~M. Medina-Mardones, Bartłomiej Lewandowski, Angus Gruen, Piotr Kucharski, and Sergei Gukov.
\newblock What makes math problems hard for reinforcement learning: a case study, August 2024.

\bibitem[Avigad(2024)]{avigad_mathematics_2024}
Jeremy Avigad.
\newblock Mathematics and the formal turn.
\newblock \emph{Bulletin of the American Mathematical Society}, 61\penalty0 (2):\penalty0 225--240, February 2024.
\newblock ISSN 0273-0979, 1088-9485.
\newblock \doi{10.1090/bull/1832}.
\newblock URL \url{https://www.ams.org/bull/2024-61-02/S0273-0979-2024-01832-1/}.

\bibitem[Bolan et~al.(2025)Bolan, Breitner, Brox, Carlini, Carneiro, Doorn, Dvorak, Goens, Hill, Husum, Mejia, Kocsis, Floch, Bar-on, Luccioli, McNeil, Meiburg, Monticone, Nielsen, Osazuwa, Paolini, Petracci, Reinke, Renshaw, Rossel, Roux, Scanvic, Srinivas, Tadipatri, Tao, Tsyrklevich, Vaquerizo-Villar, Weber, and Zheng]{bolan_equational_2025}
Matthew Bolan, Joachim Breitner, Jose Brox, Nicholas Carlini, Mario Carneiro, Floris~van Doorn, Martin Dvorak, Andrés Goens, Aaron Hill, Harald Husum, Hernán~Ibarra Mejia, Zoltan~A. Kocsis, Bruno~Le Floch, Amir~Livne Bar-on, Lorenzo Luccioli, Douglas McNeil, Alex Meiburg, Pietro Monticone, Pace~P. Nielsen, Emmanuel~Osalotioman Osazuwa, Giovanni Paolini, Marco Petracci, Bernhard Reinke, David Renshaw, Marcus Rossel, Cody Roux, Jérémy Scanvic, Shreyas Srinivas, Anand~Rao Tadipatri, Terence Tao, Vlad Tsyrklevich, Fernando Vaquerizo-Villar, Daniel Weber, and Fan Zheng.
\newblock The {Equational} {Theories} {Project}: {Advancing} {Collaborative} {Mathematical} {Research} at {Scale}, December 2025.

\bibitem[Edwards(1998)]{edwards_galois_1998}
Harold~M. Edwards.
\newblock \emph{Galois theory}.
\newblock Number 101 in Graduate texts in mathematics. Springer, New York Heidelberg, corr. 3. printing edition, 1998.
\newblock ISBN 978-0-387-90980-6 978-3-540-90980-4.

\bibitem[Matijasevič()]{Matijia2003}
Ju.~V. Matijasevič.
\newblock Enumerable sets are diophantine.
\newblock In \emph{Mathematical logic in the 20th century}, pages 269--273.
\newblock \doi{10.1142/9789812564894_0013}.

\bibitem[Gowers(2022)]{gowers_how_nodate}
W.~T. Gowers.
\newblock How can it be feasible to find proofs?, 2022.
\newblock \url {https://drive.google.com/file/d/1-FFa6nMVg18m1zPtoAQrFalwpx2YaGK4/view?usp=embed_facebook}.

\bibitem[Freedman(2025)]{freedman2025poincare}
Michael Freedman.
\newblock Millennium prize problems lecture – michael freedman: The poincaré conjecture and mathematical discovery, September 2025.
\newblock \url {https://www.youtube.com/watch?v=60X5M1FhmUc}, YouTube video.

\bibitem[Tao(2007)]{taoWHATGOODMATHEMATICS}
Terence Tao.
\newblock What is good mathematics?
\newblock \emph{Bulletin of the American Mathematical Society}, 44\penalty0 (4):\penalty0 623--634, October 2007.
\newblock \doi{10.1090/S0273-0979-07-01168-8}.

\bibitem[Euler(1758)]{euler1758elementa}
Leonhard Euler.
\newblock Elementa doctrinae solidorum.
\newblock \emph{Novi Commentarii academiae scientiarum Petropolitanae}, 4:\penalty0 109--140, 1758.
\newblock URL \url{https://scholarlycommons.pacific.edu/euler-works/230}.

\bibitem[Hatcher(2001)]{hatcher_algebraic_2001}
Allen Hatcher.
\newblock \emph{Algebraic topology}.
\newblock Cambridge university press, New York, 2001.
\newblock ISBN 978-0-521-79540-1.

\bibitem[Gallier and Xu(2013)]{gallierGuideClassificationTheorem2013}
Jean Gallier and Dianna Xu.
\newblock \emph{A Guide to the Classification Theorem for Compact Surfaces}, volume~9 of \emph{Geometry and {{Computing}}}.
\newblock Springer Berlin Heidelberg, Berlin, Heidelberg, 2013.
\newblock ISBN 978-3-642-34363-6 978-3-642-34364-3.
\newblock \doi{10.1007/978-3-642-34364-3}.

\bibitem[Dean(2014)]{dean2014homology}
Joshua Dean.
\newblock Homology using linear algebra, 2014.
\newblock Undergraduate research project, available at \href{https://osebje.famnit.upr.si/~russ.woodroofe/joshua-dean.pdf}{https://osebje.famnit.upr.si/~russ.woodroofe/joshua-dean.pdf}.

\bibitem[Wheeler(1989)]{Wheeler1989-WHEIPQ}
John~Archibald Wheeler.
\newblock Information, physics, quantum: The search for links.
\newblock In Wheeler~John Archibald, editor, \emph{Proceedings III International Symposium on Foundations of Quantum Mechanics}, pages 354--358. 1989.

\bibitem[Lewis and Lewis(1970)]{Lewis01081970}
David Lewis and Stephanie Lewis.
\newblock Holes.
\newblock \emph{Australasian Journal of Philosophy}, 48\penalty0 (2):\penalty0 206--212, 1970.
\newblock \doi{10.1080/00048407012341181}.

\bibitem[Sutton and Barto(2020)]{suttonReinforcementLearningIntroduction2020}
Richard~S. Sutton and Andrew Barto.
\newblock \emph{Reinforcement Learning: An Introduction}.
\newblock Adaptive Computation and Machine Learning. The MIT Press, Cambridge, Massachusetts London, England, second edition edition, 2020.
\newblock ISBN 978-0-262-03924-6.

\bibitem[Albrecht et~al.(2024)Albrecht, Christianos, and Sch\"afer]{marl-book}
Stefano~V. Albrecht, Filippos Christianos, and Lukas Sch\"afer.
\newblock \emph{Multi-Agent Reinforcement Learning: Foundations and Modern Approaches}.
\newblock MIT Press, Cambridge, MA, 2024.
\newblock \url{https://www.marl-book.com}.

\bibitem[Cranmer(2023)]{cranmerInterpretableMachineLearning2023}
Miles Cranmer.
\newblock Interpretable {{Machine Learning}} for {{Science}} with {{PySR}} and {{SymbolicRegression}}.jl, May 2023.

\bibitem[Riemann(1857)]{riemann_theorie_1857}
B.~Riemann.
\newblock Theorie der {Abel}'schen {Functionen}.
\newblock 1857\penalty0 (54):\penalty0 115--155, January 1857.
\newblock ISSN 1435-5345.
\newblock \doi{10.1515/crll.1857.54.115}.
\newblock URL \url{https://www.degruyterbrill.com/document/doi/10.1515/crll.1857.54.115/html}.

\bibitem[Roch(1865)]{Roch+1865+372+376}
G.~Roch.
\newblock Ueber die {Anzahl} der willkürlichen {Constanten} in algebraischen {Functionen}.
\newblock \emph{Journal für die reine und angewandte Mathematik}, 1865\penalty0 (64):\penalty0 372--376, 1865.
\newblock \doi{doi:10.1515/crll.1865.64.372}.
\newblock URL \url{https://doi.org/10.1515/crll.1865.64.372}.

\bibitem[{Kevin Buzzard}(2024)]{buzzardAIMath}
{Kevin Buzzard}.
\newblock {“Can AI do Mathematics”}, Mar 2024.
\newblock \url{https://www.youtube.com/watch?v=O0F6EFyDA58}, Talk.

\bibitem[Birch and {Swinnerton-Dyer}(1965)]{birchNotesEllipticCurves1965}
B.J. Birch and H.P.F. {Swinnerton-Dyer}.
\newblock Notes on elliptic curves. {{II}}.
\newblock \emph{crll}, 1965\penalty0 (218):\penalty0 79--108, 1965.
\newblock ISSN 0075-4102, 1435-5345.
\newblock \doi{10.1515/crll.1965.218.79}.

\bibitem[Song et~al.(2025)Song, Yang, and Anandkumar]{song_lean_2025}
Peiyang Song, Kaiyu Yang, and Anima Anandkumar.
\newblock Lean {Copilot}: {Large} {Language} {Models} as {Copilots} for {Theorem} {Proving} in {Lean}, March 2025.
\newblock arXiv:2404.12534.

\bibitem[Wu et~al.(2022)Wu, Jiang, Li, Rabe, Staats, Jamnik, and Szegedy]{wu2022autoformalization}
Yuhuai Wu, Albert~Qiaochu Jiang, Wenda Li, Markus~Norman Rabe, Charles~E Staats, Mateja Jamnik, and Christian Szegedy.
\newblock Autoformalization with large language models.
\newblock In Alice~H. Oh, Alekh Agarwal, Danielle Belgrave, and Kyunghyun Cho, editors, \emph{Advances in Neural Information Processing Systems}, 2022.
\newblock \url{https://openreview.net/forum?id=IUikebJ1Bf0}.

\bibitem[Lowe et~al.(2017)Lowe, Wu, Tamar, Harb, Abbeel, and Mordatch]{loweMultiAgentActorCriticMixed2017}
Ryan Lowe, Yi~Wu, Aviv Tamar, Jean Harb, Pieter Abbeel, and Igor Mordatch.
\newblock Multi-agent actor-critic for mixed cooperative-competitive environments.
\newblock In \emph{Proceedings of the 31st International Conference on Neural Information Processing Systems}, NIPS'17, page 6382–6393, Red Hook, NY, USA, 2017. Curran Associates Inc.
\newblock ISBN 9781510860964.

\bibitem[Sutton et~al.(1999)Sutton, McAllester, Singh, and Mansour]{SuttonPolicyGradient}
Richard~S Sutton, David McAllester, Satinder Singh, and Yishay Mansour.
\newblock Policy gradient methods for reinforcement learning with function approximation.
\newblock In S.~Solla, T.~Leen, and K.~M{\"u}ller, editors, \emph{Advances in Neural Information Processing Systems}, volume~12, Cambridge, MA, 1999. MIT Press.

\bibitem[Meyes et~al.(2019)Meyes, Lu, Puiseau, and Meisen]{meyes_ablation_2019}
Richard Meyes, Melanie Lu, Constantin Waubert~de Puiseau, and Tobias Meisen.
\newblock Ablation {Studies} in {Artificial} {Neural} {Networks}, February 2019.

\bibitem[Cowan et~al.(2011)Cowan, Cranmer, Gross, and Vitells]{cowanAsymptoticFormulaeLikelihoodbased2011}
Glen Cowan, Kyle Cranmer, Eilam Gross, and Ofer Vitells.
\newblock Asymptotic formulae for likelihood-based tests of new physics.
\newblock \emph{The European Physical Journal C}, 71\penalty0 (2):\penalty0 1554, February 2011.
\newblock ISSN 1434-6052.
\newblock \doi{10.1140/epjc/s10052-011-1554-0}.

\bibitem[{Open AI}(2026)]{openaiMAth}
{Open AI}.
\newblock Math, Oct 2026.
\newblock URL \url{https://github.com/openai/math}.
\newblock https://github.com/openai/math.

\bibitem[{Anthropic}(2026)]{anthropic2026opus48}
{Anthropic}.
\newblock Claude opus 4.8 system card.
\newblock Technical report, Anthropic, May 2026.
\newblock URL \url{https://www-cdn.anthropic.com/0b4915911bb0d19eca5b5ee635c80fef830a37ea/Claude%20Opus%204.8%20System%20Card.pdf}.
\newblock Sections 8.7 (USAMO 2026) and 8.8 (ArxivMath).

\bibitem[{Epoch AI}(2026)]{EpochFrontiermathTier4V22026}
{Epoch AI}.
\newblock {FrontierMath Tier 4 (v2)}, 6 2026.
\newblock URL \url{https://epoch.ai/benchmarks/frontiermath-tier-4-v2}.

\bibitem[Hilton(1988)]{hiltonBriefSubjectiveHistory1988}
Peter Hilton.
\newblock A {Brief}, {Subjective} {History} of {Homology} and {Homotopy} {Theory} in {This} {Century}.
\newblock \emph{Mathematics Magazine}, 61\penalty0 (5):\penalty0 282--291, December 1988.
\newblock ISSN 0025-570X, 1930-0980.
\newblock \doi{10.1080/0025570X.1988.11977391}.
\newblock URL \url{https://www.tandfonline.com/doi/full/10.1080/0025570X.1988.11977391}.

\bibitem[Bott and Tu(2008)]{bott_differential_2008}
Raoul Bott and Loring~W. Tu.
\newblock \emph{Differential forms in algebraic topology}.
\newblock Number~82 in Graduate texts in mathematics. Springer, New York, 4. print edition, 2008.
\newblock ISBN 978-0-387-90613-3 978-3-540-90613-1.

\bibitem[Gilkey and Krantz(2018)]{gilkeyInvarianceTheoryHeat2018}
Peter~B. Gilkey and Steven~George Krantz.
\newblock \emph{Invariance {Theory}: {The} {Heat} {Equation} and the {Atiyah}-{Singer} {Index} {Theorem}}.
\newblock Number v.16 in Studies in {Advanced} {Mathematics} {Ser}. Chapman and Hall/CRC, Boca Raton, 2nd ed edition, 2018.
\newblock ISBN 978-0-8493-7874-4 978-1-351-43643-4.

\bibitem[Polu et~al.(2023)Polu, Han, Zheng, Baksys, Babuschkin, and Sutskever]{polu_formal_2023}
Stanislas Polu, Jesse~Michael Han, Kunhao Zheng, Mantas Baksys, Igor Babuschkin, and Ilya Sutskever.
\newblock Formal mathematics statement curriculum learning.
\newblock In \emph{International Conference on Learning Representations (ICLR)}, 2023.
\newblock \url{https://openreview.net/forum?id=-P7G-8dmSh4}.

\bibitem[Turing(1986)]{Turing1947}
Alan~M. Turing.
\newblock Lecture to the london mathematical society on 20 february 1947.
\newblock In B.~E. Carpenter and R.~W. Doran, editors, \emph{A.M. Turing's ACE Report of 1946 and Other Papers}, volume~10 of \emph{Charles Babbage Institute Reprint Series for the History of Computing}, pages 86--105. MIT Press, Cambridge, Massachusetts, 1986.
\newblock \url{https://www.vordenker.de/downloads/turing-vorlesung.pdf}.

\bibitem[Community(2020)]{community_lean_2020}
The~mathlib Community.
\newblock The {Lean} mathematical library.
\newblock In \emph{Proceedings of the 9th {ACM} {SIGPLAN} {International} {Conference} on {Certified} {Programs} and {Proofs}}, pages 367--381, January 2020.
\newblock \doi{10.1145/3372885.3373824}.

\end{thebibliography}

\begin{appendices}
\addtocontents{toc}{\protect\setcounter{tocdepth}{0}}
\counterwithin*{table}{section}
\counterwithin*{figure}{section}
\renewcommand{\thetable}{\Alph{section}\arabic{table}}
\renewcommand{\thefigure}{\Alph{section}\arabic{figure}}
\section{Collection of statements}
\label{app:statments}

Table \ref{tab:theorems} contains some notable statements found with a range of models and datasets. Statements with essentially the same mathematical meaning occurred many times in different symbolic forms; this shows a sample. On average, each statement corresponds to one execution run. 

\medskip

\noindent \textbf{Notation Key:} For a boundary map $\partial_i$:
\[
n(\partial_i) = \text{nullity}(\partial_i), \quad r(\partial_i) = \text{rank}(\partial_i), \quad h(\partial_i) = \text{height}(\partial_i), \quad w(\partial_i) = \text{width}(\partial_i).
\]

\begin{table}[!htbp]
\centering
\small
\setlength{\tabcolsep}{3pt}
\renewcommand{\arraystretch}{2}
\makebox[\textwidth][c]{%
\begin{tabular}{c c >{\raggedright\arraybackslash}m{6.6cm} >{\centering\arraybackslash}m{4.0cm} c}
\toprule
\textbf{No.} & \textbf{Dataset} & \textbf{Statement} & \textbf{Interpretation} & \textbf{\begin{tabular}[c]{@{}c@{}}Proof\\ Feedback\end{tabular}}\\
\midrule

(1) $\mathcal{M}_0$ & $\mathcal{D}_2$ & 
$\begin{aligned}
& 
r(\partial_2)\neq n(\partial_1)-h(\partial_1)+r(\partial_1)\\ 
& \implies n(\partial_2)+h(\partial_2)\neq h(\partial_1)+w(\partial_2)
\end{aligned}$ 
& 
$V-E+F= b_2\implies b_0= b_1$&
1\\

\midrule
(2) $\mathcal{M}_0$ & $\mathcal{D}_0$ & 
$\begin{aligned}
    & n(\partial_1) + 1 = w(\partial_2) \\
    & \implies h(\partial_1)-w(\partial_1)+w(\partial_2)\neq0
\end{aligned}$ 
& 
$\begin{aligned}
\{b_0&=1\}\;  \land\;b_1 = 0\;\land\;\{b_2=1\}\\
& \implies V-E+F\neq0   
\end{aligned}
$
& 1
\\
\midrule
(3) $\mathcal{M}_0$ & $\mathcal{D}_0$ & 
$\begin{aligned}
    & r(\partial_1) = w(\partial_1) - r(\partial_2) \\
    & \implies w(\partial_1) - n(\partial_2)\neq r(\partial_2) +h(\partial_1) 
\end{aligned}$ 
& 
$\begin{aligned}
\{b_0&=1\}\; \land\;b_1 = 0\;\land\;\{b_2=1\}\\
& \implies V-E+F\neq0   
\end{aligned}
$
& 1
\\
\midrule
(4) $\mathcal{M}_0$ & $\mathcal{D}_1$ & 
$\begin{aligned}
    & n(\partial_1) = r(\partial_2)\;\land\;n(\partial_2)*w(\partial_2)=n(\partial_2) \\
    & \implies w(\partial_1)\neq h(\partial_1) + n(\partial_2) + r(\partial_2)
\end{aligned}$ 
& 
$\begin{aligned}
    \{&b_0=1\}\;\land\;b_1=0\\ &\;\land\;b_2*\text{dim}(C_2)=b_2\\ &\implies V-E+F\neq0   
\end{aligned}
$
& 1
\\
\midrule
(5) $\mathcal{M}_2$ & $\mathcal{D}_2$ & 
$\begin{aligned}
    & n(\partial_1)=r(\partial_2)\;\land\;n(\partial_2)=1\\
    &\implies n(\partial_2)=h(\partial_1)-r(\partial_1)
\end{aligned}$ 
& 
$\begin{aligned}
    &b_1=0\;\land b_2=1\\
    & \implies b_0=b_2
\end{aligned}$
& 1
\\
\midrule
(6) $\mathcal{M}_2$ & $\mathcal{D}_2$ & 
$\begin{aligned}
    & n(\partial_1)=r(\partial_2)\;\land\;n(\partial_2)=1\\
    &\implies h(\partial_1)-r(\partial_1)=1
\end{aligned}$ 
& 
$\begin{aligned}
    &b_1=0\;\land\;  b_2=1\\
     & \implies b_0=1
\end{aligned}$
& 1
\\
\midrule

(7) $\mathcal{M}_0$ & $\mathcal{D}_0$ & 
$\begin{aligned}
    & (w(\partial_2) = n(\partial_1) - n(\partial_2)) \land (r(\partial_2) = n(\partial_1)) \\
    & \implies (w(\partial_1) + 1 - h(\partial_1) = r(\partial_2)) \\
    & \quad \quad \land (n(\partial_1) = w(\partial_2) + 1)
\end{aligned}$ 
& 
$\begin{aligned}
\{&b_0=1\}\;\land\;b_1=0\;\land\;\\ &b_1=2\;\land\;\{b_2=1\}\\ \implies &V-E+F = 2\;\land\;b_1=2   
\end{aligned}$

& 1
\\

\midrule
(8) $\mathcal{M}_0$ & $\mathcal{D}_0$ & 
$\begin{aligned}
    n(\partial_1) = r(\partial_1) -h(\partial_1)+w(\partial_2)
\end{aligned}$ 
& 
$b_2 - 1=b_1$&
0\\
\midrule
(9) $\mathcal{M}_0$ & $\mathcal{D}_2$ & 
$\begin{aligned}
    h(\partial_1)-r(\partial_1)=n(\partial_2)
\end{aligned}$ 
& 
$b_0=b_2$&
0\\
\midrule

(10) $\mathcal{M}_0$ & $\mathcal{D}_2$ & 
$\begin{aligned}
    w(\partial_1)=h(\partial_2)
\end{aligned}$ 
& 
$E=E$&
Trivially 1\\
\bottomrule
\end{tabular}%
}

\vspace{0.2pt}
\caption{A collection of statements, showing the raw symbolic form, a mathematical interpretation, and the model and dataset that produced it. In the fifth column, 0/1 represents whether the statement was proven/not proven in Lean. In the Interpretation column, atomic formulae contained in curly braces \{.\} were in the premise set for that experiment. We also allow minor rewrites (i.e. substitutions) from the premise set corresponding to the execution. During a run, the system will typically pass through many false statements, and some may even receive reward, such as (8) and (9). Number (8) is true on spheres but false on tori. Recall both $\chi:=V-E+F$ and $\chi:=b_0-b_1+b_2$.
}
\label{tab:theorems}
\end{table}

Learning problem \ref{lp1} is completed by statements (1) - (4), which find and combine both definitions of the Euler characteristic. Although the interpretations of (2) and (3) are the same, the forms found by the system are different. Statements
(5) and (6) are also somewhat noteworthy, without using $\chi$. Finally, although (1) - (7) are proven by the system using linear algebra, some may be best (and easily) understood using the Classification theorem and/or Poincar\'e duality \cite{gallierGuideClassificationTheorem2013,hatcher_algebraic_2001}. Statements (8)-(10) are representative cases where the system fails. These include trivial statements, such as $E=E$ in (10) but also more interesting failures such as the relations between Betti numbers in (8) and (9).

\begin{figure}[!htbp]
    \centering
    \makebox[\textwidth][c]{%
        \includegraphics[width=\linewidth]{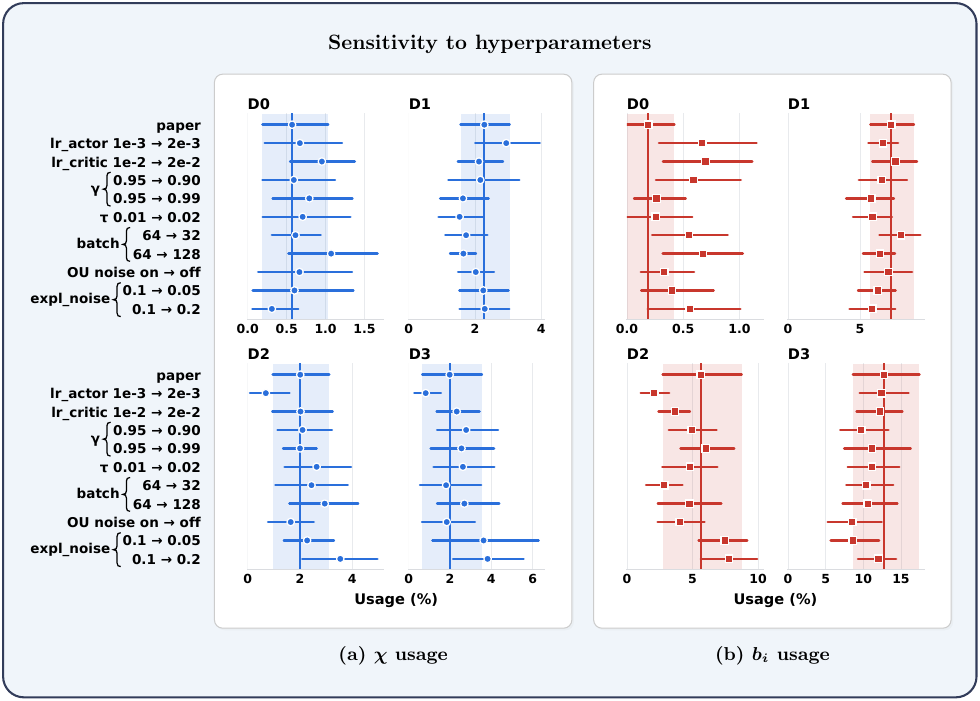}%
    }    
    \caption{The figure shows the results of hyperparameter sensitivity analysis. We test the effects of varying the following parameters: learning rates of the actor and critic (lr\textunderscore actor, lr\textunderscore critic), the discount factor $\gamma$ (x2), the batch size (x2), Ornstein-Uhlenbeck noise (OU noise) and the exploration noise (expl\textunderscore noise) (x2). Each change was averaged over ten episodes and two random seed choices. The vertical lines represent the values of the parameters chosen in the paper with the shaded columns giving the 95\% bootstrap confidence intervals. The rates of $\chi$ and $b_i$ are within these in all cases, besides the use of $b_i$ in $\mathcal{D}_0$, where even better choices could have been made.}
    \label{fig:hyperparam}
\end{figure}

We also test the sensitivity of our results for the full model $\mathcal{M}_0$ over all datasets to the choice of RL hyperparameters (Figure \ref{fig:hyperparam}), varying the parameters one at a time whilst keeping the others fixed. The prevalence of statements with $\chi$ and/or $b_i$ is fairly consistent and largely within the confidence intervals of the values chosen in the paper. This suggests that the raw numbers of §\ref{sec:D0results}-\ref{sec:results_varying_data} are not at a local maximum. 

\section{LLM baseline experiments}
\label{app:LLM_baseline}

Table \ref{tbl:llm_baselines} contains the fill results of the LLM baseline experiments. To avoid contamination, tests were conducted in a sandbox without access to the internet. We tested performance with and without `tool use'; the data was shared as a readable file in the first case and passed directly into the prompt in the second. Perhaps counter-intuitively, access to computation in Python did not improve performance.
{\footnotesize
\setlength{\tabcolsep}{3pt}
\renewcommand{\arraystretch}{1}
\setlength{\LTcapwidth}{\textwidth}

\begin{longtable}{c *{10}{w{c}{0.075\textwidth}}}

\caption{Full performance of the LLMs across all prompts. To control for model stochasticity, we performed $10$ runs for each prompt: $5$ without tool calls (column `A') and 5 with tool calls (column `B').
In (i) prompts p1-p6 provide various suggestive hints all similar to $p3$ in the main paper. Then as in Table \ref{tbl:llm_baselines_main}, (ii) separates out the especially suggestive hint $p7$. The symbol `*' represents `unsatisfactory partial fulfilment', runs that only mention the relevant concept in passing.}
\label{tbl:llm_baselines} \\
\toprule
\multirow{2}{*}{Prompts}
& \multicolumn{2}{c}{Opus 4.6} 
& \multicolumn{2}{c}{Opus 4.8}
& \multicolumn{2}{c}{Fable 5.2}
& \multicolumn{2}{c}{GPT-5.2 Pro}
& \multicolumn{2}{c}{GPT-6.2-Astra} \\
\cmidrule(lr){2-3} \cmidrule(lr){4-5} \cmidrule(lr){6-7} \cmidrule(lr){8-9} \cmidrule(lr){10-11}
& A & B & A & B & A & B & A & B & A & B \\
\midrule
\endfirsthead

\toprule
\multirow{2}{*}{Prompts}
& \multicolumn{2}{c}{Opus 4.6} 
& \multicolumn{2}{c}{Opus 4.8}
& \multicolumn{2}{c}{Fable 5.2}
& \multicolumn{2}{c}{GPT-5.2 Pro}
& \multicolumn{2}{c}{GPT-6.2-Astra} \\
\cmidrule(lr){2-3} \cmidrule(lr){4-5} \cmidrule(lr){6-7} \cmidrule(lr){8-9} \cmidrule(lr){10-11}
& A & B & A & B & A & B & A & B & A & B \\
\midrule
\endhead

\bottomrule
\endlastfoot

\multicolumn{11}{l}{\textbf{(i)}} \\
\midrule

p1 
& 0/5 & 0/5
& 0/5 & 0/5
& 1/5 & 2/5
& 1*/5 & 0/5
& 0/5 & 0/5 \\
\midrule

p2 
& 0/5 & 0/5
& 0/5 & 0/5
& 3/5 & 0/5
& 0/5 & 0/5
& 0/5 & 0/5 \\
\midrule

p3
& 0/5 & 0/5
& 0/5 & 0/5
& 3/5 & 0/5
& 0/5 & 0/5
& 0/5 & 0/5 \\
\midrule

p4 
& 0/5 & 0/5
& 0/5 & 0/5
& 1*/5 & 0/5
& 0/5 & 0/5
& 0/5 & 0/5 \\
\midrule

p5
& 0/5 & 0/5
& 0/5 & 0/5
& (1*+2)/5 & 1/5
& 0/5 & 0/5
& 0/5 & 0/5 \\
\midrule

p6 
& 0/5 & 0/5
& 0/5 & 0/5
& 2/5 & 0/5
& 0/5 & 0/5
& 0/5 & 0/5 \\
\midrule

\multicolumn{11}{l}{\textbf{(ii)}} \\
\midrule

p7
& 1/5 & 2/5
& 0/5 & 1/5
& 5/5 & 5/5
& 0/5 & 0/5
& 3*/5 & (3*+1)/5 \\

\end{longtable}
}

As discussed in the caption, in this Table we indicate unsatisfactory partial fulfilments with the symbol `*'. These were outputs that mentioned the relevant human-aligned concepts in a very incomplete or passing way. For example, GPT 6.2-Astra gave the following response to $p7$:

\noindent\fbox{\parbox{\dimexpr\linewidth-2\fboxsep-2\fboxrule}{%
Let $x1,\dots,x8$ denote the columns in order. A first-order formula over the integers that holds on every row is
\begin{equation*}
    \exists n,\exists c;\Bigl[ (c=0\lor c=2)\ \land
(x_1,\ldots,x_8) =(3n,2n,n+c,3n,n+c-1,2n-1,2n+1-c,1) \Bigr]    
\end{equation*}
Tuple equality abbreviates the conjunction of the eight coordinate equalities.
The interesting invariant is $c=x3-x1+x2$, an alternating sum of the form $V-E+F$ used for the Euler characteristic. It takes only the values 0 and 2, and the seventh column records the same choice in reverse: $x7=x2+1-c$. Thus $x3+x7=x1+1$ regardless of which value $c$ takes.
The two rows beginning with 939,626 show the distinction at the same $n=313$: one has $(x3,x7)=(313,627)$, while the other has (315,625). The alternating sum changes from 0 to 2, but the pair sums to 940 in both rows.
}}
\\\\
Compared to the successful response of Fable 5.2 in the main text, this is clearly unsatisfactory, providing no interpretation using the relevant algebraic structures, e.g. chain complexes. This should also be seen in the light of the model's extensive knowledge of the concepts. A large portion of * fulfilments occurred for the highly suggestive $p7$, for which it is clearly dissatisfying.  

\section{Formal translation}
\label{app:LeanTranslation}
\begin{figure}[!htbp]
    \vspace{-0.7cm}
    \centering
    \makebox[\textwidth][c]{%
        \includegraphics[width=\linewidth]{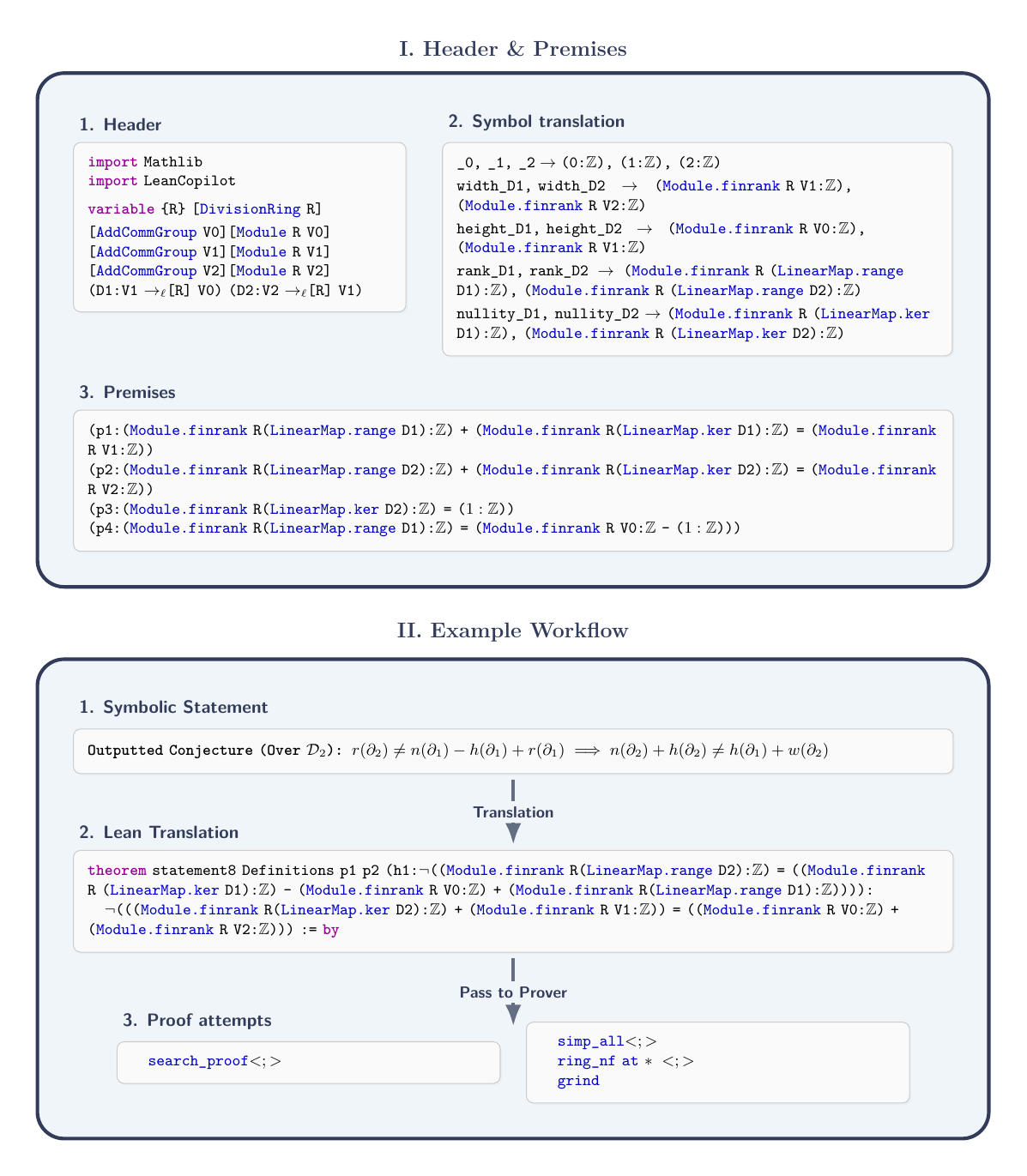}%
    }
    
    \caption{ In Subfigure I, the header code defines \texttt{V0}, \texttt{V1} and \texttt{V2} as $\texttt{R}$-modules over a division ring \texttt{R}, and $\partial_i$s (represented via \texttt{D1} and \texttt{D2}) as $\texttt{R}$-linear maps. \texttt{p1}-\texttt{p4} are the rank-nullities, orientability and connectedness premises.
    Subfigure II translates statement \textbf{(1)} over dataset $\mathcal{D}_2$, picking out $\texttt{p1}$, $\texttt{p2}$ and the rank-nullities. Step 3 shows the two provers we tried: on the left, Lean Copilot's \texttt{search\_proof}, and on the right, the pre-written proof using \texttt{grind} and \texttt{ring}, tactics that are suited to algebra. In our experiments we saw no qualitative difference between these provers.
    }
    \label{fig:Lean4Translation}
\end{figure}
This section summarises the proving part of the environment. Statements output by the $\CA$ are initially translated into Lean. Firstly, this adds boilerplate code to define the relevant variables, e.g. defining vector spaces and linear maps in a way that is standard in mathlib \cite{community_lean_2020}. 
Secondly, the symbolic statement is translated into a Lean theorem, with the experiment's premise set becoming the hypotheses.

With this formulation, proving a given true statement becomes an algebraic manipulation of the relevant ranks of modules.
Figure \ref{fig:Lean4Translation} shows all the parts of the translation code, as well as the full proving process for statement \textbf{(1)} in Table \ref{tab:theorems}.




\end{appendices}

\end{document}